\documentclass[10pt,journal,transmag]{IEEEtran}

\ifCLASSOPTIONcompsoc
  \usepackage[nocompress]{cite}
\else
  \usepackage{cite}
\fi

\usepackage[caption=false,font=normalsize,labelfont=sf,textfont=sf]{subfig}
\usepackage{textcomp}
\usepackage{stfloats}
\usepackage{url}
\usepackage{verbatim}
\usepackage{amsthm,amsmath,amssymb}
\usepackage{algorithm}
\usepackage{algorithmic}
\usepackage{color}
\usepackage{float}
\usepackage{booktabs}
\usepackage{ulem}
\usepackage[mathscr]{eucal}
\usepackage{graphicx}
\usepackage{amsfonts}
\usepackage{balance}
\usepackage{array}
\usepackage{titlesec}

\newtheorem{lemma}{Lemma}

\def\BibTeX{{\rm B\kern-.05em{\sc i\kern-.025em b}\kern-.08em
    T\kern-.1667em\lower.7ex\hbox{E}\kern-.125emX}}

\newcolumntype{L}[1]{>{\raggedright\let\newline\\\arraybackslash\hspace{0pt}}m{#1}}
\newcolumntype{C}[1]{>{\centering\let\newline\\\arraybackslash\hspace{0pt}}m{#1}}
\newcolumntype{R}[1]{>{\raggedleft\let\newline\\\arraybackslash\hspace{0pt}}m{#1}}

\ifCLASSINFOpdf
  \graphicspath{{./tnseFigs/}}
\else
  \graphicspath{{./tnseFigs/}}
\fi

\begin{document}

\title{Training Latency Minimization for Model-Splitting Allowed Federated Edge Learning}

    \author{Yao Wen, Guopeng Zhang, Kezhi Wang, and Kun Yang 

    \vspace{-3mm}

    \IEEEcompsocitemizethanks{
        \IEEEcompsocthanksitem Yao Wen and Guopeng Zhang are with the School of Computer Science and Technology, China University of Mining and Technology, Xuzhou 221116, China. E-mail: ywen@cumt.edu.cn; gpzhang@cumt.edu.cn.
        \IEEEcompsocthanksitem Kezhi Wang is with the Department of Computer Science, Brunel University London, Middlesex UB8 3PH, U.K. E-mail: kezhi.wang@brunel.ac.uk.
        \IEEEcompsocthanksitem Kun Yang is with the School of Computer Science and Electronic Engineering, University of Essex, Colchester CO4 3SQ, U.K. E-mail: kunyang@essex.ac.uk. 
    }
}

% \markboth{Journal of \LaTeX\ Class Files,~Vol.~14, No.~8, August~2015}%
% {Shell \MakeLowercase{\textit{et al.}}: Bare Demo of IEEEtran.cls for Computer Society Journals}

\IEEEtitleabstractindextext{

\begin{abstract}

To alleviate the shortage of computing power faced by clients in training deep neural networks (DNNs) using federated learning (FL), we leverage the \textit{edge computing} and \textit{split learning} to propose a model-splitting allowed FL (\texttt{SFL}) framework, with the aim to minimize the training latency without loss of test accuracy. Under the \textit{synchronized global update} setting, the latency to complete a round of global training is determined by the maximum latency for the clients to complete a local training session. Therefore, the training latency minimization problem (TLMP) is modelled as a minimizing-maximum problem. To solve this mixed integer nonlinear programming problem, we first propose a \textit{regression method} to fit the quantitative-relationship between the \textit{cut-layer} and other parameters of an AI-model, and thus, transform the TLMP into a continuous problem. Considering that the two subproblems involved in the TLMP, namely, the \textit{cut-layer selection problem} for the clients and the \textit{computing resource allocation problem} for the parameter-server are relative independence, an alternate-optimization-based algorithm with polynomial time complexity is developed to obtain a high-quality solution to the TLMP. Extensive experiments are performed on a popular DNN-model \textit{EfficientNetV2} using dataset MNIST, and the results verify the validity and improved performance of the proposed \texttt{SFL} framework.

\end{abstract}

}

\maketitle

\IEEEdisplaynontitleabstractindextext

\IEEEpeerreviewmaketitle

% \IEEEraisesectionheading{\section{Introduction}}
\section{Introduction}

\IEEEPARstart{T}{he} latest Artificial Intelligence (AI) products are powered by cutting-edge machine learning (ML) technology, ranging from face detection \cite{Yan2022PrivacypreservingOA} and speech recognition \cite{9746585} installed on mobile devices to virtual assistants deployed in autonomous systems \cite{9985482}.
Large-scale data is essential for training high-performance AI-models, e.g., decision trees, support vector machines (SVMs), and deep neural networks (DNNs).
However, centralized approach to training AI-models requires clients to transfer privately-owned data to servers, which poses a great threat to users' privacy and security \cite{9810958}.
Google has proposed a new ML paradigm, called federated learning (FL) \cite{McMahan2017CommunicationEfficientLO}, that allows multiple clients to train an AI-model in a distributed manner, while keeping their data local. The vanilla FL algorithm, called \texttt{FedAvg} is illustrated in Fig. \ref{Split Federated Learning process details} (the left part). A parameter server (PS) first distributes an AI-model to be trained to multiple selected clients. Then, each client trains the model locally with the relevant data and upload the trained model to the PS for aggregation.
This process is iterated for several rounds and the aggregated model achieves a certain test accuracy.

\begin{figure}[htbp]
    \centering
    \includegraphics[width=8.8cm,height=7.8cm]{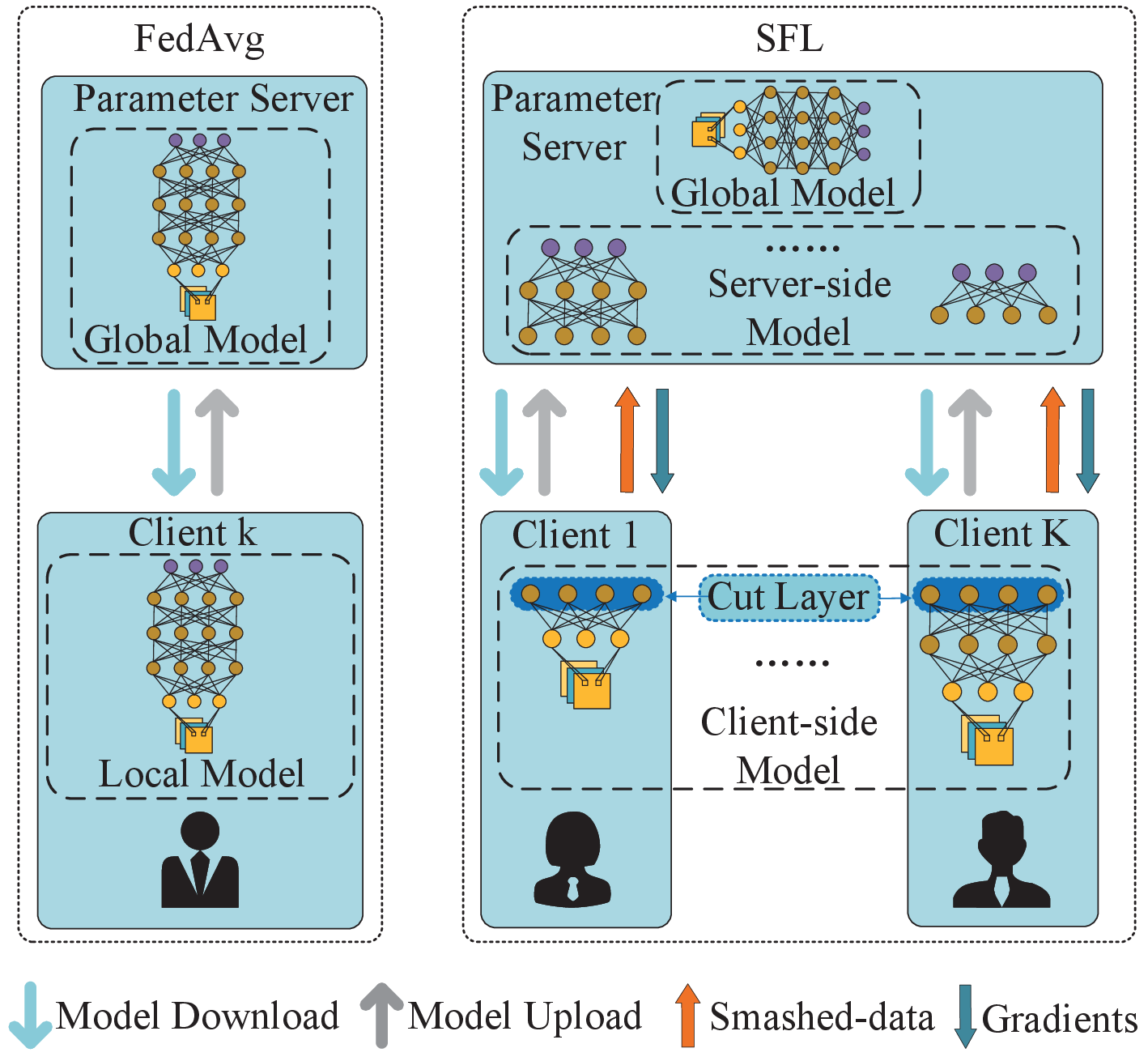}
    \caption{The frameworks of \texttt{FedAvg} and \texttt{SFL}.}
    \label{Split Federated Learning process details}
\end{figure}

However, \texttt{FedAvg} is more suitable for training lightweight AI-models, where the communication cost is greater than the computational cost caused to clients \cite{McMahan2017CommunicationEfficientLO}.
With the rapid popularization of \textit{high-speed mobile communications} (such as 5G) and the development of \textit{large-scale deep AI-models} (such as DNNs), the application context of FL has undergone fundamental changes:
\begin{itemize}
    \item The ultra-reliable and low-latency communication technologies of 5G make the communication cost no longer the bottleneck for FL \cite{8972389}.
    \item The computing power of mobile devices has not been improved significantly, compared with the rapidly growing demand for training large-scale DNN-models \cite{10134466}.
\end{itemize}
New techniques and methods are needed to cope with the new trends in the development of AI-models.

Fortunately, the method of \textit{split learning} (SL) \cite{GUPTA20181} supported by \textit{edge computing} \cite{GUPTA20181} provides a feasible solution to the above challenges. 
As illustrated in Fig. \ref{Split Federated Learning process details} (the right part), the model-split allowed FL (\texttt{SFL}) first splits an AI-model into two parts, which can be trained separately [4]. The part trained by the client is called the \textit{client-side} model, while the other part trained by the PS is called the \textit{server-side} model. The client initiates the forward propagation (FP) on the \textit{client-side} model to obtain the \textit{intermediate result}, called the \textit{smashed-data}, and then, uploads it to the PS. The PS uses the received \textit{smashed-data} as the input to continue the FP on the \textit{server-side} model.
Thereafter, the PS can initiate the backward propagation (BP) on the \textit{server-side} model. The obtained \textit{gradients}, as the \textit{server-side} \textit{intermediate result}, is transferred to the client, so the client can use it to complete the BP on the \textit{client-side} model.

Notably, by equipping the PS with a high-performance computing server, it can then process multiple \textit{server-side} models (offloaded by multiple clients) in parallel, thus greatly reducing the computational burden of the clients.
However, unlike in \texttt{FedAvg}, the clients in \texttt{SFL} cannot independently perform the FP and BP, resulting in the \textit{intermediate results} being communicated hundreds of times during training \textit{client-side} models.
The reduced computational burden on clients may lead to the increased network data traffic.
Although this fits the trend of increased network bandwidth and larger-and-deeper AI-models, the following issues must be addressed in order to improve the training efficiency of \texttt{SFL}:

\textit{Challenge 1: Cut-layer selection of the clients}. In the paradigm of SL, the \textit{cut-layer} for a client refers to the last layer of the \textit{client-side} model \cite{9812907}, as shown in Fig. \ref{Split Federated Learning process details}. The client-selected \textit{cut-layer} determines its communication load (for transmitting the \textit{smashed data}) and computational load (for training the \textit{client-side} model).
However, there is no explicit relationship between the \textit{cut-layer} and other parameters of an AI-model, which brings to a client a huge discrete solution space and extremely high computational complexity to choose the optimal \textit{cut-layer}.
    
\textit{Challenge 2: Computing resource allocation of the PS}. In \texttt{SFL}, a PS can simultaneously train multiple \textit{server-side} models offloaded by the participating clients. Due to the \textit{heterogeneity} of clients in \textit{computing power} and \textit{datasets}, the PS must optimize its allocation of the limited computing resources to improve the training efficiency of \texttt{SFL}.

In this paper, we leverage \textit{edge computing} and \textit{split learning} to improve the \textit{training efficiency} of \texttt{SFL}. In particular, we aim to minimize the training latency of the \texttt{SFL} using \textit{synchronized global model update} (SGMU) \cite{9778210} without losing the test accuracy of the trained model. 
Under the SGMU setting, the latency in completing a round of global training is determined by the maximum latency for the participating clients to complete a session of local training. 
Therefore, the training latency minimization problem (TLMP) is modelled as one that minimizes the maximum latency for the clients to complete a local training session.
To address \textit{Challenge 1} inherent in the problem, we take a popular AI-model, \textit{EfficientNetV2} \cite{tan2021efficientnetv2}, as an example and propose a \textit{regression method} to fit the quantitative relationship between the \textit{cut-layer} of the trained model and the resulting \textit{communication} and \textit{computational} loads for a client. Then, the original mixed integer nonlinear programming (MINLP) problem is transformed into a continuous one. Considering the relative independence of \textit{cut-layer selection} and \textit{computing resource allocation}, this problem is further decoupled into two subproblems, namely, the \textit{cut-layer selection problem of the clients} (to address \textit{Challenge 1}) and the \textit{computing resource allocation problem of the PS} (to address \textit{Challenge 2}). Finally, an alternate-optimization-based algorithm is proposed to obtain a high-quality solution to the training latency minimization problem.
Extensive experiments are performed on \textit{EfficientNetV2} using dataset MNIST \cite{deng2012mnist}, and the results verify the validity of our proposed method.

In summary, the main technical contributions of this paper are as follows:
\begin{enumerate}
    \item The \textit{edge computing} and \textit{split learning} techniques are orchestrated to improve the training efficiency of \texttt{SFL}.
    With the specific aim to minimize the training latency of the \texttt{SFL} using SGMU without loss of test-accuracy, an MINLP problem, called the TLMP, is formulated. The TLMP achieves the goal by constantly removing the system bottleneck, i.e., the maximum latency for the participating clients to complete a local training session.

    \item To address \textit{Challenge 1} inherent in the TLMP, a \textit{regression method} is presented to fit the quantitative relationship between the \textit{cut-layer} of an AI-model and the \textit{communication} and \textit{computational} loads generated to a client. In this way, the original MINLP problem is transformed into a continuous one.

    \item By decoupling the continuous TLMP into two independent subproblems, namely, the \textit{cut-layer selection problem of the clients} and the \textit{computing resource allocation problem of the PS}, an alternate-optimization-based algorithm with polynomial time complexity is designed to obtain a high-quality solution to the original TLMP. 
    Extensive experiments testify the effectiveness of the proposed method.       
\end{enumerate}

The rest of this paper is organized as follows. Sec. \uppercase\expandafter{\romannumeral2} reviews the related works. Sec. \uppercase\expandafter{\romannumeral3} describes the detailed training procedure of the \texttt{SFL} using SGMU. Sec. \uppercase\expandafter{\romannumeral4} first gives the parameterized expressions for the time consumption of the \texttt{SFL}, and then, formulates the TLMP as a min-max problem. The algorithm to solve the TLMP is developed in Sec. \uppercase\expandafter{\romannumeral5}. In Sec. \uppercase\expandafter{\romannumeral6}, experiment results are provided to verify the effectiveness of the proposed method. Finally, the paper is summarized in Sec. \uppercase\expandafter{\romannumeral7}.

\section{Related Works}
\textit{Federated Learning:}
As an effective distributed machine learning method for privacy protection, FL has been widely studied and applied in many fields.
Parameter Server architecture \cite{10.5555/2999134.2999271} is widely used to enable a large number of computation nodes to train a shared model by aggregating locally-computed updates.
Federated optimization  \cite{DBLP:journals/corr/KonecnyMRR16} is proposed in to improve communication efficiency by minimizing communication rounds, while keeping training data local and protecting data privacy.
To make FL more practical, an effective client selection algorithm is proposed in \cite{8761315} to solve the straggler (clients requiring longer training times) problem. Issues of fairness assurance between participating clients were discussed in \cite{9272649}. In \cite{9155494}, the authors counterbalanced the bias introduced by non-IID data and accelerated the convergence of model training. In \cite{WOS:000871080800013}, local training of clients constrained by computing power is accelerated by training task offloading. In \cite{9261995}, the authors studied the non-convex resource allocation problem of FL over wireless networks.

\textit{Splitting Learning :}
Distributed training of large-scale AI-models requires large amount of computing resources, but the limited computing resources of clients become the bottleneck restricting the performance (e.g., training latency, test accuracy, etc.) of FL.
Thereby, some recent researches focus on reducing the computational burden of clients. Split learning (SL) \cite{DBLP:journals/corr/abs-1812-00564} is one of the methods, which can divide an AI-model into multiple parts and trains them separately in a certain order.
A parallel SL method is proposed in \cite{9016486} to prevent overfitting due to differences in the training order and data size of the segmented model parts.
In \cite{8871124}, the authors use SL to assist the collaborative computation of DNNs between mobile devices and cloud server, and formulate the optimal computing-resource scheduling problem for the DNN layers as the shortest path problem and integer linear programming.
In \cite{9814499}, the authors proposed a multi-split algorithm that can assign DNN submodels to each computational node in a given network topology, and a computational graph search problem was proposed to optimize the assignment.
In \cite{9569707}, SL is used to reduce the total energy cost of edge devices under time-varying wireless channels.
In \cite{9923620}, the coexistence of FL and SL is allowed in wireless networks, and the convergence of hybrid split and federated learning (HSFL) algorithm is analyzed with non-IID data distribution.

\textit{Collaborative training of AI-models:}
In a collaborative training framework such as FL, the bottleneck is in the computing power and workload differences between the client and the edge server. Therefore, some recent studies have focused on the allocation of computing, storage and communication resources for edge devices to improve the efficiency of collaborative training.
To eliminate the straggler clients in an FL task, a timeout threshold is set in \cite{DBLP:journals/corr/abs-1804-08333}, after which the uploaded model will be discarded to reduce the overall training latency.
An effective method is proposed in \cite{9134409} to dynamically and virtually allocate the computing and communication resources of an edge server to multiple clients, allowing the task offloading of the clients. The challenge lies in estimating client training time, as the real time consumption is only known after the actual run.
To address this issue, the data fitting method, i.e., optimal regression model, is proposed in \cite{9812907} to predict training time of clients. 
The authors in \cite{9272649} modelled the estimation of client's local time consumption (consisting of transmitting and training model) as a C$^2$MAB problem.
Then, they converted the online scheduling problem into an offline problem with Lyapunov optimization and solved it using a divide-and-conquer method.
By applying deep reinforcement learning (DRL), a capability-matched model offloading strategy is designed in \cite{9778210} for heterogeneous clients, aiming to minimize the training latency of an FL task.
The authors in \cite{9999679} used DNN partitioning to minimize the FL training latency under the constraints of device-specific participation rates, energy consumption, and memory usage, regardless of the transmission of the smashed data over a wireless network.

\textit{Summary:} Inspired by the existing works, we leverage \textit{edge computing} and SL to improve the training efficiency of FL, a collaborative training paradigm. Under the synchronized
global model update (SGMU) setting, the two key challenges described in Sec. I, namely, (1) \textit{cut-layer selection of the clients} and (2) \textit{computing resource allocation of the PS}, must be dealt with jointly, which is not yet involved in the related works.

\section{System Model}

We consider a FL system consisting of a unique PS and a set $\mathcal{K}$ of $K$ clients.
The goal of the PS is to train an AI-model with the data owned by the clients but without directly sharing the data. Next, we review \texttt{FedAvg}, the vanilla FL algorithm, and then, describe \texttt{SFL}, the \textit{edge computing} and \textit{split learning} supported FL method. The system parameters are summarized in the following Table \ref{tab:parameter}.

\begin{table}[!htbp]
\newcommand{\tabincell}[2]{\begin{tabular}{@{}#1@{}}#2\end{tabular}}  
\begin{center}
    \caption{Parameter Table}
    \label{tab:parameter}
    \begin{tabular}{cp{2.54in}}
        \toprule
        Symbol & Description \\
        \midrule
        
        $\textbf{w}_k$, $\textbf{w}^\text{C}_k$, $\textbf{w}^\text{S}_k$ & Local model of client $k$, \textit{client-side} model of client $k$, \textit{server-side} model of client $k$. \\
        
        $\mathcal{D}_k$, $\mathcal{B}_k$ & Local dataset of client $k$, mini-batch of client $k$. \\
        
        $x_{k, n}$, $y_{k, n}$, $\hat{y}_{k, n}$ & The $n^\text{th}$ sample of $\mathcal{D}_k$ or $\mathcal{B}_k$, ground-true label, prediction of the label.\\
        
        $\mathcal{L}(\cdot)$, $\triangledown \mathcal{L}(\cdot)$ & Sample-wise loss function, gradient of loss function.\\
        
        $\eta$ & Learning rate.\\
        
        $I_k$ & Number of epochs in a local training session of client $k$.\\

        $w^{(l)}$ & The parameters of the $l^\text{th}$ layer of model $\textbf{w}$. \\

        $L$ & The number of layers of an AI-model. \\
        
        $l_k$, $l_k^\text{min}$ & \textit{Cut-layer} of client $k$, the minimum value of $l_k$ that client $k$ can take. \\
        
        $\mathcal{S}_{k,n}$ & Sample-wise \textit{smashed-data} of client $k$. \\
        
        $\mathcal{G}_{k,n}$ & Sample-wise \textit{gradients} of client $k$. \\
        
        $\Lambda_k$ & Size of $\mathcal{S}_{k,n}$ and $\mathcal{G}_{k,n}$. \\

        $r_k$ & Data rate between the PS and client $k$. \\
        
        $f_k^\text{C}$  & Computing power of client $k$ for training 
        $\textbf{w}^\text{C}_k$. \\
        
        $f_k^\text{S}$ & Computing power of the PS allocated to client $k$ for training $\textbf{w}^\text{S}_k$.\\

        $F^\text{C}_k$, $B^\text{C}_k$ & Amount of computing resources required for client $k$ to perform FP/BP with one data-sample.\\
        
        $F^\text{S}_k$, $B^\text{S}_k$ & Amount of computing resources required for the PS to perform FP/BP with one data-sample.\\

        $F_k^\text{tot}$ & $F_k^\text{tot}=F^\text{C}_k+B^\text{C}_k$, which is the computational load of client $k$ for training the \textit{client-side} model.\\ 
        
        $\Gamma$ & Amount of computing resources required to train an AI-model with one data-sample. \\
        
        $F^{\text{max}}$ & Amount of computing resources available for the PS. \\

        $T_k$ & Latency for client $k$ to complete a local training session. \\
        
        $\mathcal{T}$ & Latency for the $K$ clients to complete a round of global training. \\
        
        $\textbf{L}$ & $\textbf{L}=(l_1, \cdots, l_K)$. \\
        
        $\textbf{F}$ & $\textbf{F}=(f^\text{S}_1, \cdots, f^\text{S}_k, \cdots, f^\text{S}_K)$. \\
        
        $\tilde{T}_k$ & Latency for client $k$ to complete a local training session using only the local computing power $f_k^\text{C}$. \\
        
        $\tilde{k}$ & New index of client $k$ according to $\tilde{T}_k$. \\
        
        $\mathcal{K}_\texttt{FedAvg}$ & Set of clients adopting $\texttt{FedAvg}$ to train the local model.\\
        
        $\mathcal{K}_\text{\texttt{SFL}}$ & Set of clients adopting $\texttt{SFL}$ to train the local model.\\
        
        $\theta$ & The index of the first client in $\mathcal{K}_\text{\texttt{SFL}}$. \\

        \bottomrule
    \end{tabular}
\end{center}
\end{table}

\subsection{The Basic of \texttt{FedAvg}}
In \texttt{FedAvg} \cite{McMahan2017CommunicationEfficientLO}, the PS first broadcasts the initial model $\textbf{w}$ to the $K$ clients. Then, each client $k$ ($\forall k \in \mathcal{K}$) trains the model, represented by $\textbf{w}_k$, locally and independently by using its own dataset $\mathcal{D}_k = \{\left(x_{k,n}, y_{k,n}\right)| n = {1, \cdots, n_k}\}$, where $n_k$ is the size of the dataset, and $x_{k,n}$ and $y_{k,n}$ are respectively the sample and its ground-true label.

\textit{Local training at client k}: Using the \textit{stochastic gradient descent} (SGD), a mini-batch $\mathcal{B}_k\subseteq \mathcal{D}_k$ can be randomly sampled from $\mathcal{D}_k$ to train $\textbf{w}_k$ in each local training epoch. With each $(x_{k,n}, y_{k,n}) \in \mathcal{B}_k$, client $k$ first performs the FP to obtain $\hat{y}_{k,n}=\texttt{fp}(x_{k,n};\textbf{w}_k)$, the prediction of $y_{k,n}$ using the current model $\textbf{w}_k$. 
Then, client $k$ performs the BP and calculates the gradient $\triangledown \mathcal{L}(\hat{y}_{k,n}, y_{k,n}; \textbf{w}_k)$ w.r.t $\textbf{w}_k$, where $\mathcal{L}(\hat{y}_{k,n}, y_{k,n}; \textbf{w}_k)$ is the sample-wise loss function. The local model of client $k$ can be updated as
\begin{align}
    \textbf{w}_k = \textbf{w}_k - \eta \triangledown \mathcal{L}(\hat{y}_{k,n}, y_{k,n}; \textbf{w}_k),\ \forall k \in \mathcal{K},
\end{align}
where $\eta$ is the learning rate.

During a local training session, the above training epoch is iterated $I_k$ times. Thereafter, client $k$ uploads the trained model $\textbf{w}_k$ to the PS for aggregation.

\textit{Global model update at the PS}: Under the \textit{synchronized global model update} (SGMU) setting, the PS can update the global model $\textbf{w}$ by using the following eq. \eqref{agrregation function}, only if all the local models of the $K$ clients are collected. 
\begin{equation}
    \textbf{w} = \sum_{k\in\mathcal{K}}\frac{n_k}{\sum_{k\in\mathcal{K}}n_k}\textbf{w}_k. \label{agrregation function}
\end{equation}

The global model update will be iterated several rounds until the learning goal (e.g., a certain prediction accuracy on the test dataset or the maximum number of global training rounds) is achieved.

\subsection{The Framework of \texttt{SFL}}
Referring to \cite{9812907}, a \textit{layer} of an AI-model is defined as the minimum divisible unit of the model parameter, which can be either a separate layer (for example, active, convolution, or fully connection layers) or a combination of multiple consecutive layers.
Then, an AI-model $\textbf{w}$ can be partitioned into $L$ layers as
\begin{equation}
    \textbf{w}=w^{(1)} \uplus \cdots \uplus w^{(l)}\uplus \cdots \uplus w^{(L)},
\end{equation}
where $w^{(l)}$ represents the $l^\text{th}$ layer's parameter of the model and the operator $"\uplus"$ represents connecting any two consecutive layers.

Unlike \texttt{FedAvg}, which orders each client to train $\textbf{w}_k$ independently, \texttt{SFL} splits $\textbf{w}_k$ for each client $k$ into the following two parts
\begin{equation}
    \textbf{w}_k=\textbf{w}^\text{C}_k \uplus \textbf{w}^\text{S}_k,\  \forall k \in \mathcal{K},
    \label{model integrity}
\end{equation}
where $\textbf{w}^\text{C}_k = w^{(1)} \uplus \cdots \uplus w^{(l_k)}$ and $\textbf{w}^\text{S}_k = w^{(l_k+1)} \uplus \cdots \uplus w^{(L)}$ are respectively called the \textit{client-side} model and \textit{server-side} model, and are respectively submitted to client $k$ and the PS for training. The last layer of the \textit{client-side} model, $w^{(l_k)}$, is termed as the \textit{cut-layer} for client $k$ \cite{Wu2022SplitLO}.

\begin{figure}[htbp]
    \centering
    \includegraphics[width=9cm,height=8.9cm]{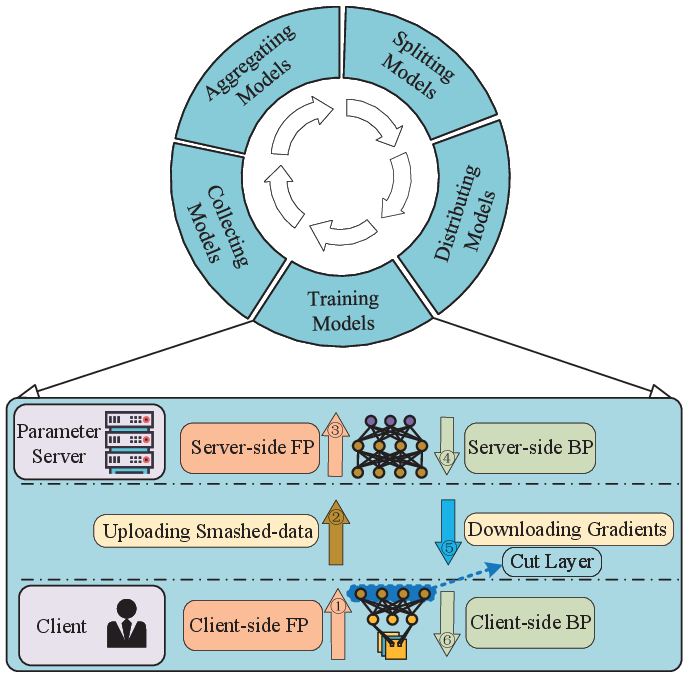}
    \caption{The workflow of the \texttt{SFL}.}
    \label{Split Federated Learning process show}
\end{figure}

In Fig. \ref{Split Federated Learning process show}, we show the workflow of \texttt{SFL}, which includes the following 5 phases: 1) Splitting local models; 2) Distributing local models; 3) Training local models; 4) Collecting local models; and 5) Aggregating local models. The main difference between \texttt{SFL} and \texttt{FedAvg} is in the $3^\text{rd}$ and $5^\text{th}$ phases. The detail is given below.

\textbf{Phase 3} (Training local models): Since $\textbf{w}_k$ is split into two parts, namely $\textbf{w}^\text{C}_k$ and $\textbf{w}^\text{S}_k$, the training of $\textbf{w}_k$ requires the collaboration between client $k$ and the PS. 
\begin{itemize}
    \item \textit{Forward Propagation}:
    \begin{itemize}
        \item[(1)] \textit{Client-side FP}: In each local training epoch, client $k$ starts the FP on the \textit{client-side} model $\textbf{w}^\text{C}_k$ using mini-batch $\mathcal{B}_k$. The output at the cut-layer $w^{(l_k)}$, called the \textit{smashed-data} \cite{9812907}, is given as
        \begin{equation}
            \mathcal{S}_{k,n} = \texttt{fp}\big(x_{k,n}; \textbf{w}^\text{C}_k\big),\  \forall k \in \mathcal{K}.
        \end{equation}
        
        \item[(2)] \textit{Uploading smashed data}: Client $k$ uploads the \textit{smashed data} $\mathcal{S}_{k,n}$ and the label $y_{k,n}$\footnote{Uploading labels to PS undermines the data-privacy of clients. To address this issue, the authors in \cite{10.1145/3559613.35632011111} proposed a three-stage SL method to avoid the leakage. For reconstruction attacks on training data, the leakage risk can also be reduced by using the differential privacy \cite{9609550}\cite{10097035}. This paper uses an approach similar to \cite{9812907} that allows a client to independently select the minimum \textit{cut-layer} to implement its security requirement. Please refer to eq. \eqref{UE runing time in SFL fill with model} and its description in Sec. IV-B for detail.} as the client-side \textit{intermediate results}, to the PS for further processing. 
    
        \item[(3)] \textit{Server-side FP}: Upon receiving $\mathcal{S}_{k,n}$ and $y_{k,n}$, the PS continues the FP on the \textit{server-side} model $\textbf{w}^\text{S}_k$, and obtains the predicted value of $y_{k,n}$ as 
        \begin{equation}
            \hat{y}_{k,n} = \texttt{fp}\big(\mathcal{S}_{k,n}; \textbf{w}^\text{S}_k\big),\  \forall k \in \mathcal{K}.
        \end{equation}
    \end{itemize}

    \item \textit{Backward propagation}:
    
    \begin{itemize}
        \item[(4)] \textit{Sever-side BP}: With $\hat{y}_{k,n}$ and $y_{k,n}$, the PS instead of client $k$ starts the BP. The \textit{server-side} model $\textbf{w}^\text{S}_k$ is updated as
        \begin{equation}
            \textbf{w}^\text{S}_k = \textbf{w}^\text{S}_k - \eta \triangledown \mathcal{L}(\hat{y}_{k,n}, y_{k,n}; \textbf{w}^\text{S}_k),\  \forall k \in \mathcal{K}.
        \end{equation}

        \item[(5)] \textit{Downloading gradients}: Let $\mathcal{G}_{k,n}$ denote the \textit{gradients} generated at layer $w^{l_{k+1}}$, the first layer of $\textbf{w}^\text{S}_k$. The PS should transfers $\mathcal{G}_{k,n}$ as the server-side \textit{intermediate results} to client $k$ to continue the BP.
    
        \item[(6)] \textit{Client-side BP}: Upon receiving $\mathcal{G}_{k,n}$ from the PS, client $k$ can complete the \textit{client-side} BP and updates the \textit{client-side} model $\textbf{w}^\text{C}_k$ as 
        \begin{equation}
            \textbf{w}^\text{C}_k = \textbf{w}^\text{C}_k - \eta \triangledown \mathcal{L}(\mathcal{G}_{k,n}; \textbf{w}^\text{C}_k),\  \forall k \in \mathcal{K}.
        \end{equation}
    \end{itemize}
\end{itemize}

\textbf{Phase 5} (Aggregating local models): At the end of each local training session, client $k$ uploads the trained \textit{client-side} model $\textbf{w}^\text{C}_{k}$ to the PS. The PS can then combine $\textbf{w}^\text{C}_{k}$ with the trained \textit{server-side} model $\textbf{w}^\text{S}_{k}$ as $\textbf{w}_k = \textbf{w}^\text{C}_k \uplus \textbf{w}^\text{S}_k$.
Under the SGMU setting, when obtaining all $\textbf{w}_k$ of the $K$ clients, the PS can aggregate them to update the global model $\textbf{w}$ by using eq. \eqref{agrregation function}.

\section{Problem Formulation}
Although the \texttt{SFL} reduces the computational burden of clients for training local models, the communication overhead inevitably increases because the \textit{intermediate results} $\mathcal{S}_{k,n}$ and $\mathcal{G}_{k,n}$ are communicated multiple times between the PS and each client due to the collaborative training.

In this section, we first quantify the latency of each phase of the \texttt{SFL}. Then, the training latency minimization problem is proposed.

\subsection{Latency of each phase}
Referring to Sec. III-B, the latency of each phase of the \texttt{SFL} is analyzed as follows.

\textbf{Phase 1} (Splitting local models): The latency of this phase is caused by performing \textbf{Algorithm \ref{Optimization of problem with one iteration}}, namely, the joint cut-layer selection and computing resource allocation algorithm proposed in Sec. V-E. Considering that the algorithm has a polynomial time complexity and runs on the PS with sufficient computing power, the latency is negligible.

\textbf{Phase 2} (Distributing local models): Let $r_k$ denote the data rate between the PS and client $k$ in the current round of global training \footnote{In general, the uplink bandwidth and downlink bandwidth are asymmetric. Since the \textit{intermediate results} generated by the clients are different, in the downlink, PS can only distribute the \textit{intermediate results} to the clients using orthogonal unicast rather than broadcast. So we assume that the uplink and downlink rate is the same for each client.}. The latency for client $k$ to download the \textit{client-side} model $\textbf{w}^\text{C}_k$ from the PS is given by
\begin{equation}
   \texttt{DM}_k =  |\textbf{w}^\text{C}_k|/{r_k}, \forall k \in \mathcal{K}, \label{tau_sfl_k}
\end{equation}
where $|\textbf{w}^\text{C}_k|$ denotes the size (in bits) of $\textbf{w}^\text{C}_k$.

\textbf{Phase 3} (Training local models): As analyzed in Sec. III-B, the training of the local model of any client $k$, $\textbf{w}_k=\textbf{w}^\text{C}_k \uplus \textbf{w}^\text{S}_k$, consists of 6 stages. The latency of each stage is given below. 
\begin{itemize}
    \item[(1)] \textit{Client-side} FP. Let $F^\text{C}_k$ denote the number of \textit{float point operations} (Flops) required for client $k$ to perform the sample-wise FP on $\textbf{w}^\text{C}_k$. Let $f_k^\text{C}$ (in Flops/s) denote the available computing power for client $k$. Then, the latency to process the total $|\mathcal{B}_k|$ samples is given by
    \begin{equation}
        \texttt{FP}_k^\text{C} = F^\text{C}_k |\mathcal{B}_k| / {f_k^\text{C}},\ \forall k \in \mathcal{K}. 
    \end{equation}

    \item[(2)] \textit{Uploading smashed-data}. Let $|\mathcal{S}_{k,n}|$ denoted the size (in bits) of $\mathcal{S}_{k,n}$. The latency of transmitting the total $|\mathcal{B}_k|$ pieces of \textit{smashed-data} is given by\footnote{To start the BP, the PS must have the label of each sample. The labels can be transferred by the clients to the PS along with the \textit{smashed-data}. Since the size of labels is much smaller than the \textit{smashed-data}, the transmission overhead is omitted here.}
    \begin{equation}
        \texttt{TS}_k = |\mathcal{S}_{k,n}| |\mathcal{B}_k| / r_k,\  \forall k \in \mathcal{K}.
    \end{equation}
       
    \item[(3)] \textit{Server-side FP}. Let $F^\text{S}_k$ denote the number of Flops required for the PS to perform the sample-wise FP on $\textbf{w}^\text{S}_k$. Let $f_k^\text{S}$ (in Flops/s) denote the computing power allocated by the PS to client $k$. Then, the latency of processing the total $|\mathcal{B}_k|$ samples is given by
    \begin{equation}
        \texttt{FP}_k^\text{S} = F^\text{S}_k |\mathcal{B}_k| / f_k^\text{S},\ \forall k \in \mathcal{K}.
    \end{equation}
    
    \item[(4)] \textit{Server-side BP}. Let $B^\text{S}_k$ denote the number of Flops required for the PS to perform the sample-wise BP on $\textbf{w}^\text{S}_k$. Then, the latency of processing the total $|\mathcal{B}_k|$ samples is given by
    \begin{equation}
        \texttt{BP}_k^\text{S} = B^\text{S}_k |\mathcal{B}_k| / f_k^\text{S},\ \forall k \in \mathcal{K}. 
    \end{equation}

    \item[(5)] \textit{Downloading gradient}. Let $|\mathcal{G}_{k,n}|$ (in bits) denote the size of $\mathcal{G}_{k,n}$. Then, the latency of downloading the total $|\mathcal{B}_k|$ \textit{gradients} is given by
    \begin{align}
        \texttt{TG}_k = |\mathcal{G}_{k,n}| |\mathcal{B}_k| / {r_k} , \forall k \in \mathcal{K}.
    \end{align}   

    \item[(6)] \textit{Client-side BP}. Let $B^\text{C}_k$ denote the number of Flops required for client $k$ to perform the sample-wise BP on $\textbf{w}^\text{C}_k$. 
    The time required to process the total $|\mathcal{B}_k|$ samples is given by
    \begin{equation}
        \texttt{BP}_k^\text{C} = B^\text{C}_k  |\mathcal{B}_k| / f_k^\text{C},\ \forall k \in \mathcal{K}.
    \end{equation}
\end{itemize}

\textbf{Phase 4} (Collecting client models): After $I_k$ local training epochs, client $k$ uploads the updated \textit{client-side} model $\textbf{w}^\text{C}_k$ to the PS for aggregation. The latency is given by 
\begin{equation}
   \texttt{UM}_k =  |\textbf{w}^\text{C}_k| / r_k ,\ \forall k \in \mathcal{K}. \label{tau_sfl_k_upload_model}
\end{equation}

\textbf{Phase 5} (Aggregating local models):
Aggregating client models requires only small computation effort. Since the PS has sufficient computing power, the latency is negligible.

\subsection{Overall time consumption of \texttt{SFL}}
In general, the \textit{smashed-data} and \textit{gradients} generated by processing one data-sample have the same size $|\mathcal{S}_{k,n}|=|\mathcal{G}_{k,n}|=\Lambda_k$. The latency for client $k$ to complete a local training session, i.e., $I_k$ local training epochs, is given by
\begin{align}
    T_k & = \texttt{DM}_k + \texttt{UM}_k + \notag \\
    &\ \ \ \ I_k |\mathcal{B}_k| (\texttt{FP}_k^\text{C} + \texttt{TS}_k + \texttt{FP}_k^\text{S} + \texttt{BP}_k^\text{S} + \texttt{TG}_k + \texttt{BP}_k^\text{C}) \notag \\
    & = 2\frac{|\textbf{w}^\text{C}_k|}{r_k} + I_k |\mathcal{B}_k| \left( \frac{ F^\text{C}_k+B^\text{C}_k}{f_k^\text{C}} + \frac{ F^\text{S}_k+B^\text{S}_k}{f_k^\text{S}} + 2\frac{\Lambda_k}{r_k} \right), \notag \\
    & \;\;\;\;\; \forall l_k \in \{l_k^\text{min}, \cdots, L\}, \ \forall k \in \mathcal{K}. \label{UE runing time in SFL fill with model} 
\end{align}
where $l_k^\text{min}$ is the minimum value of $l_k$ that client $k$ can take. The value of $l_k^\text{min}$ reflects the privacy requirement of client $k$. The larger $l_k^\text{min}$, the less likely it is to derive the raw information of the client from $\mathcal{S}_{k,n}$.

From eq. \eqref{UE runing time in SFL fill with model}, we note that when the \textit{cut-layer} is selected as $l_k=L$, $\textbf{w}^\text{C}_k=\textbf{w}$ and $\textbf{w}^\text{S}_k = \emptyset$. It means that client $k$ prefers to train the entire model $\textbf{w}_k$ locally, as in \texttt{FedAvg}. In this case, one can get
\begin{align}
    F^\text{S}_k=0,\ B^\text{S}_k=0, \ \text{and}\ \Lambda_k=0,\ \text{if}\ l_k=L,\ \forall k \in \mathcal{K}.
    \label{fedavg mode}
\end{align}
By substituting eq. \eqref{fedavg mode} into eq. \eqref{UE runing time in SFL fill with model} , the latency for client $k$ to complete a local training session in this case is given by 
\begin{align}
    T_k = 2\frac{|\textbf{w}|}{r_k}+ I_k |\mathcal{B}_k| \frac{\Gamma}{f^\text{C}_k},\ \text{if}\ l_k=L,\ \forall k \in \mathcal{K}, \label{Time consumption T_k at layer L}
\end{align}
where $|\textbf{w}|$, $\Gamma$ are the size of model $\textbf{w}$ (in bits) and total amount of computing load for one data-sample of model $\textbf{w}$.

Combining eqs. \eqref{UE runing time in SFL fill with model} and \eqref{Time consumption T_k at layer L}, we can represent the training latency $T_k$ of client $k$ in the following form.
\begin{equation}
    T_k =  
    \begin{cases} 
        \displaystyle {\frac{2|\textbf{w}^\text{C}_k|}{r_k} + I_k |\mathcal{B}_k| \left( \frac{ F^\text{C}_k+B^\text{C}_k}{f_k^\text{C}} + \frac{ F^\text{S}_k+B^\text{S}_k}{f_k^\text{S}} + 2\frac{\Lambda_k}{r_k} \right)},\\ 
        \quad \quad \quad \quad \quad \quad\ \forall l_k \in \{l_k^\text{min}, \cdots, L-1\},\ \forall k \in \mathcal{K}, \\
        \displaystyle {\frac{2|\textbf{w}|}{r_k}+ I_k |\mathcal{B}_k| \frac{\Gamma}{f^\text{C}_k}},\ \text{if}\ l_k=L,\ \forall k \in \mathcal{K}.\\
    \end{cases}
    \label{original expression for T_k}
\end{equation}

\subsection{Training Latency Minimization Problem}
From eq. \eqref{UE runing time in SFL fill with model}, we know that
\begin{itemize}
    \item The latency for any client $k$ to train the model and communicate the \textit{intermediate results} depends on the selected cut-layer $l_k$;
    \item While the PS is much more powerful than each of the clients, it needs to serve many clients at the same time, so its available computing resources are relatively limited. 
\end{itemize}
Hence, the communication and computing resources of the clients and PS must be jointly scheduled to adapt to the \textit{cut-layers} selected for the clients, so that the overall training latency of the \texttt{SFL} can be minimized.

Because we can only accurately predict the available computing resources and data rate for the clients and PS for a short time in the future \cite{9778210}, our goal is set to minimize $\mathcal{T}$, the latency to complete one round of global training. As the SGMU is adopted, the overall latency $\mathcal{T}$ depends on $\mathcal{T}=\max_{k\in\mathcal{K}}T_k$, the maximum latency for the $K$ clients to complete a local training session. Let $\textbf{L} = (l_1, \cdots, l_K)$ and $\textbf{F} = (f^\text{S}_1, \cdots, f^\text{S}_k, \cdots, f^\text{S}_K)$. The training latency minimization problem for the \texttt{SFL} is formulated as
\begin{align}
    & \min_{\textbf{L},\textbf{F}} \mathcal{T}, \; \;  \mathcal{T}=\max_{k\in\mathcal{K}}\ T_k,  \label{minT_total_original:mainfun mathcal T by lf} \\
    \mbox{s.t.}\;\;\;
    & l_k \in \{ l_k^\text{min}, \cdots, \ L\},\ \forall k \in \mathcal{K}, \tag{\ref{minT_total_original:mainfun mathcal T by lf}.1} \label{minT_total_original:subject l_m range} \\
    & \sum^{K}_{k=1} f^\text{S}_k \leqslant F^{\text{max}},  \tag{\ref{minT_total_original:mainfun mathcal T by lf}.2} \label{minT_total_original:subject_f_max} \\
    & |\textbf{w}^\text{C}_k| + |\textbf{w}^\text{S}_k| = |\textbf{w}|,\ \forall k \in \mathcal{K}, \tag{\ref{minT_total_original:mainfun mathcal T by lf}.3} \label{minT_total_original:sum_of_model_transmission_size} \\
    & F^\text{C}_k + F^\text{S}_k + B^\text{S}_k + B^\text{C}_k = \Gamma,\ \forall k \in \mathcal{K}, \tag{\ref{minT_total_original:mainfun mathcal T by lf}.4} \label{minT_total_original:sum_of_model_computation}
\end{align} where constraint \eqref{minT_total_original:subject l_m range} is the value space of the \textit{cut-layer} for client $k$,
constraint \eqref{minT_total_original:subject_f_max} limit the maximum amount of computing resources available for the PS to $F^{\text{max}}$,
constraint \eqref{minT_total_original:sum_of_model_transmission_size} comes from eq. \eqref{model integrity}, which guarantees the model integrity after being split for any client $k$, and constraint \eqref{minT_total_original:sum_of_model_computation} comes from the following fact.

Given the AI-model to be trained, the amount of computing resources required to train the model with one data-sample, represented as $\Gamma$, and the amounts of computing resources required to perform the sample-wise BP and FP, respectively represented as $F^\text{tot}$ ($F^\text{tot} = F^\text{C}_k + F^\text{S}_k$) and $B^\text{tot}$ ($B^\text{tot}= B^\text{S}_k + B^\text{C}_k$), are constants. Therefore, we have the following equation, i.e., constraint \eqref{minT_total_original:sum_of_model_computation}. 
\begin{equation}
    \Gamma = F^\text{tot} +B^\text{tot}=F^\text{C}_k + F^\text{S}_k + B^\text{S}_k + B^\text{C}_k,\ \forall k \in \mathcal{K}. \label{min_T:sum_of_model_computation} \\
\end{equation} 

Since both continuous variable $\textbf{F}$ and discrete variable $\textbf{L}$ are involved, problem \eqref{minT_total_original:mainfun mathcal T by lf} is a mixed integer nonlinear programming (MINLP) problem, which cannot be solved directly by using conventional methods. In what follows, we develop effective methods to address this problem.

\section{Solving the Problem}
In problem \eqref{minT_total_original:mainfun mathcal T by lf}, the size of the \textit{client-side} model ($|\textbf{w}^\text{C}_k|$), the computing resources needed to train the \textit{client-side} model ($F_k^\text{tot}=F_k^\text{C}+B_k^\text{C}$), and the amount of generated \textit{intermediate results} ($\Lambda_k$) all depend on the selected cut-layer $l_k$ for client $k$.
However, as shown in the following Figs. \ref{fig:Model size fitting effect}, \ref{fig:Model computation fitting effect}, and \ref{fig:Smashed output size fitting effect}, there is no explicit relationship between the \textit{cut-layer} and other parameters of an AI-model. 
Since a DNN usually has hundreds of layers, this will lead to a unusually high time and space complexity to find the optimal \textit{cut-layers} for the clients. 
To the best of our knowledge, there is currently no better way to address this issue. In this paper, we propose a \textit{regression method} to quantify the relationship between the \textit{cut-layer} and other parameters of a given AI-model. Based on that, a fast and effective method is developed to find the high-quality solution of problem \eqref{minT_total_original:mainfun mathcal T by lf}.

It is worth noting that \textit{logistic regression} and \textit{curve fitting} methods have been widely used to address such relationship fitting problems. For example, in \cite{Boshkovska2015PracticalNE} and \cite{9354851}, the authors built practical non-linear energy harvesting models by curve fitting for measurement data. Based on the fitted model, effective resource allocation algorithms are developed for wireless information and power transfer (SWIPT) systems.

\subsection{Fit The Relationship Between The Parameters of An AI-model}
Different AI-models have different neural network structures.
To the best of our knowledge, there is currently no better way to obtain their relationship expression.
In this paper, we use the popular AI-model \textit{EfficientNetV2}\cite{tan2021efficientnetv2} as an example and propose the \textit{regression method} to fit the relationship between the \textit{cut-layer} and other parameters of an AI-model.

\subsubsection{Client-side model size against different cut-layers}
As shown in Fig. \ref{fig:Model size fitting effect}, with increasing $l_k$, the increase of $|\textbf{w}^\text{C}_k|$ is not obvious at the beginning, but the growth rate increases sharply when $l_k>36$. Therefore, we set the relationship between $l_k$ and $|\textbf{w}^\text{C}_k|$ as
\begin{align}
    |\textbf{w}^\text{C}_k| = & \alpha (l_k)^2,\ \forall k \in \mathcal{K} \label{model size fitting function}, 
\end{align}
where $\alpha \geq 0$ is the parameter to be fitted.
\begin{figure}[htbp]
    \centering
    \includegraphics[width=7cm,height=5.3cm]{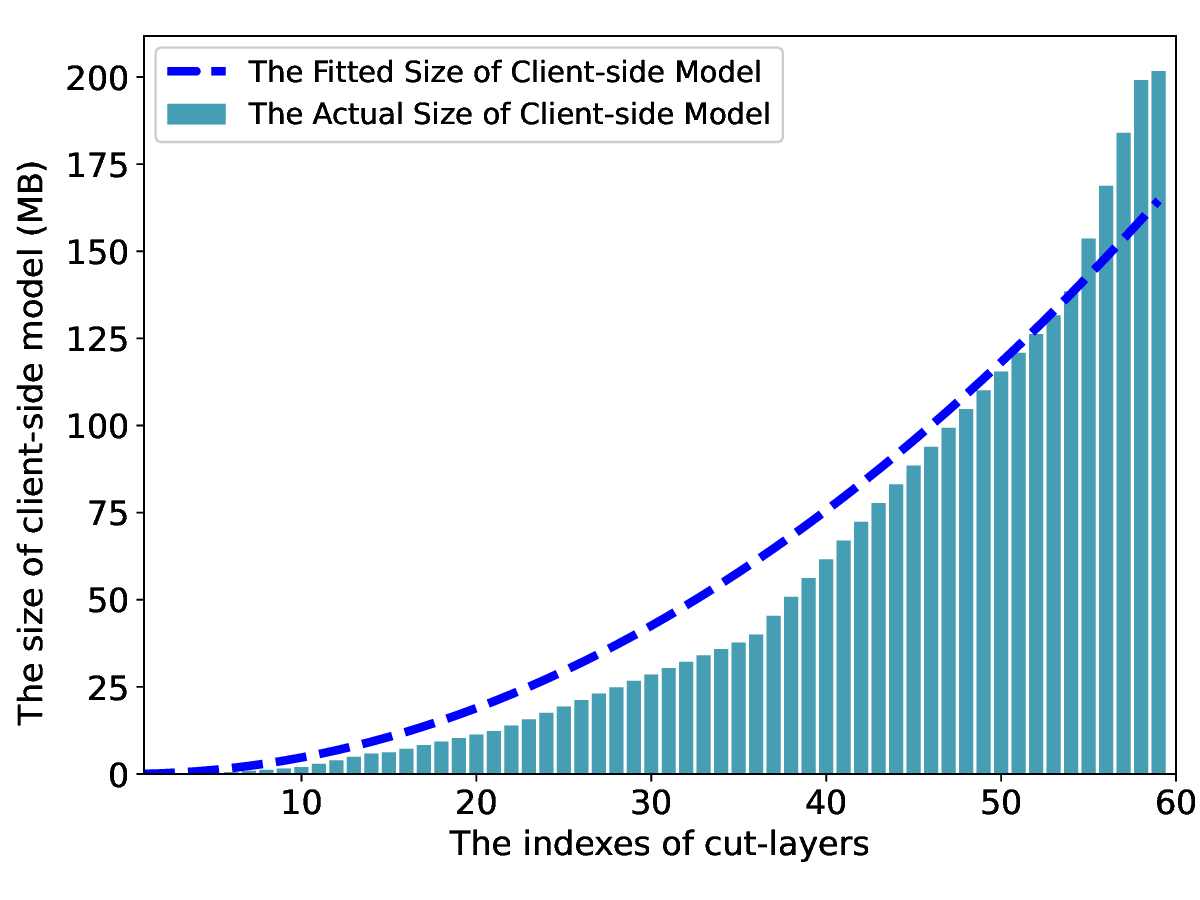}
    \caption{The quantitative-relationship between $l_k$ and $|\textbf{w}^\text{C}_k|$.}
    \label{fig:Model size fitting effect}
\end{figure}

\subsubsection{Training load of client against different cut-layers}
The training of the \textit{client-side} model includes the FP and BP.
The computational load of performing the FP can be obtained by using the python package \textit{torchinfo} \cite{9825633}, however, there is no direct way to know the computational load of performing the BP.
To address this issue, we trained \textit{EfficientNetV2} 1500 times using a dataset of size 32. The time consumed by FP and BP in each training session is shown in Fig. \ref{The time consumption for FP and BP.}.

\begin{figure}[htbp]
    \centering
    \includegraphics[width=7cm,height=4cm]{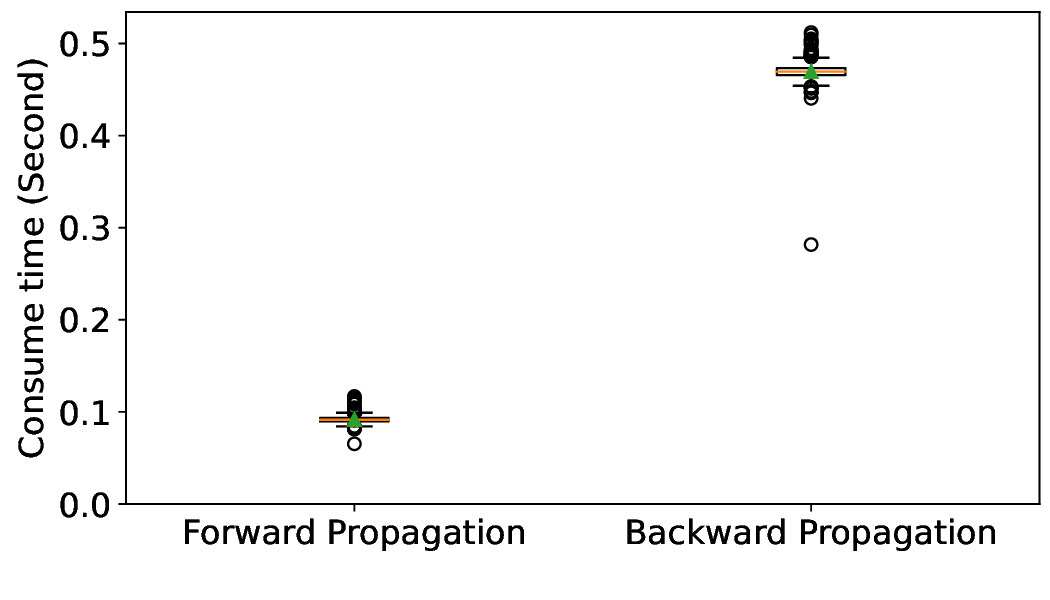}
    \caption{The time consumption of FP and BP}
    \label{The time consumption for FP and BP.}
\end{figure}

From Fig. \ref{The time consumption for FP and BP.}, we see that with the same computing power the time consumed by BP is roughly several times the time consumed by FP. So we have
\begin{equation}
    B_k^\text{tot} \approx \kappa F_k^\text{tot},\ \forall k \in \mathcal{K},
\end{equation}
where $\kappa \geqslant 1$. Fig. \ref{fig:Model computation fitting effect} shows the computational load of client $k$ for training the \textit{client-side} model (that is, $F_k^\text{tot}=F^\text{C}_k+B^\text{C}_k$) with different \textit{cut-layers}. There is an approximate linear relationship between them, which is given by
\begin{equation}
    F_k^\text{tot}=F^\text{C}_k+B^\text{C}_k = \beta l_k (1 + \kappa ),\ \forall k \in \mathcal{K},
    \label{kapa}
\end{equation}
where $\beta > 0$ is the parameter to be fitted.
\begin{figure}[htbp]
    \centering
    \includegraphics[width=7cm,height=5.3cm]{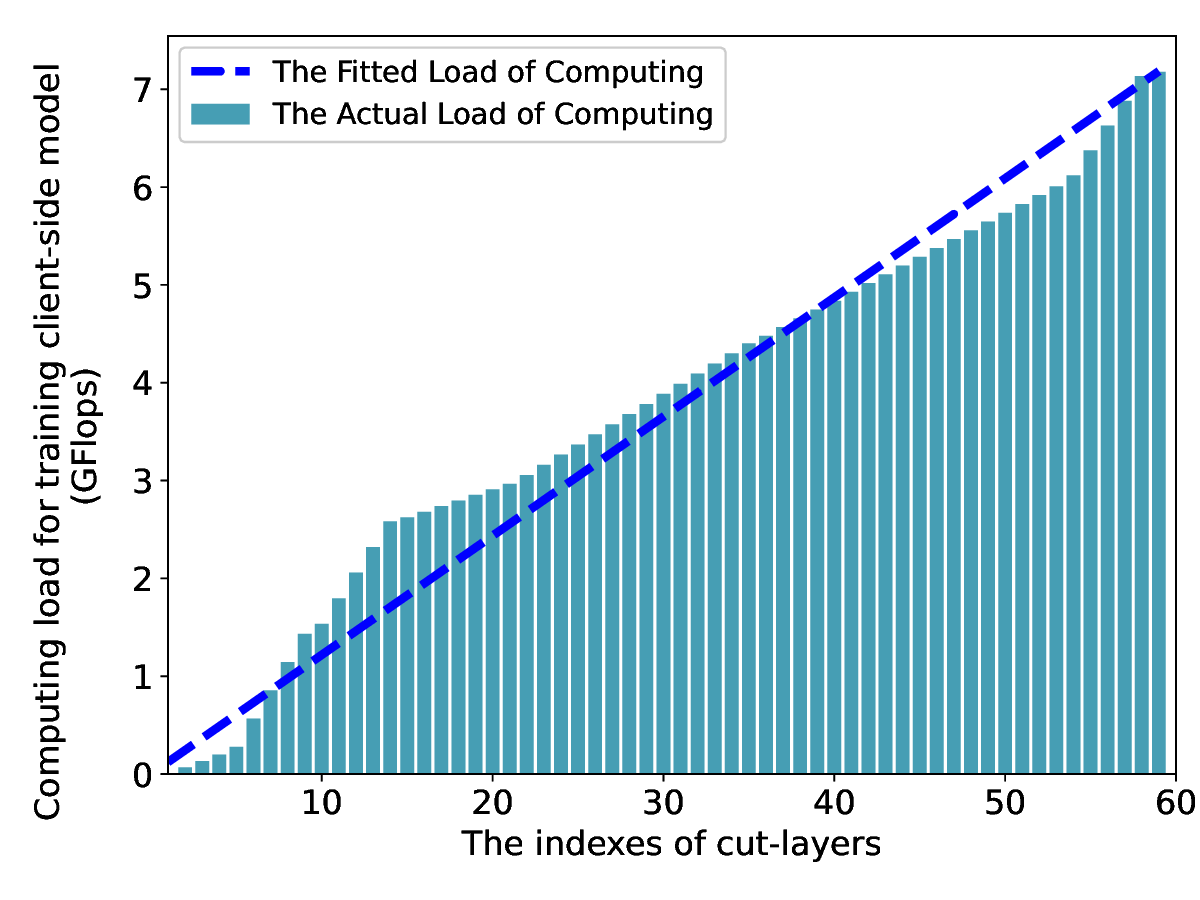}
    \caption{The quantitative-relationship between $l_k$ and $F_k^\text{tot}$.}
    \label{fig:Model computation fitting effect}
\end{figure}

\subsubsection{Size of intermediate results against different cut-layers} The head of a DNN is always a convolutional neural network (CNN) to extract features, thus reducing the computational burden of the subsequent feedforward neural network (FNN). Hence, the size of the \textit{intermediate results} decreases rapidly in the CNN layers but changes slowly in the FNN layers, as shown in Fig. \ref{fig:Smashed output size fitting effect}.
Accordingly, we set the relationship between $l_k$ and $\Lambda_k$ as
\begin{equation}
    \Lambda_k = \left|\mathcal{S}_{k,n}\right| =\left|\mathcal{G}_{k,n}\right| = \frac{\gamma_1}{l_k+\gamma_2},\ \forall k \in \mathcal{K},
    \label{smashed-data size fitting function}
\end{equation}
where $\gamma_1 > 0$ and $ \gamma_2 \geqslant 0$ are the parameters to be fitted.
\begin{figure}[htbp]
    \centering
    \includegraphics[width=7cm,height=5.3cm]{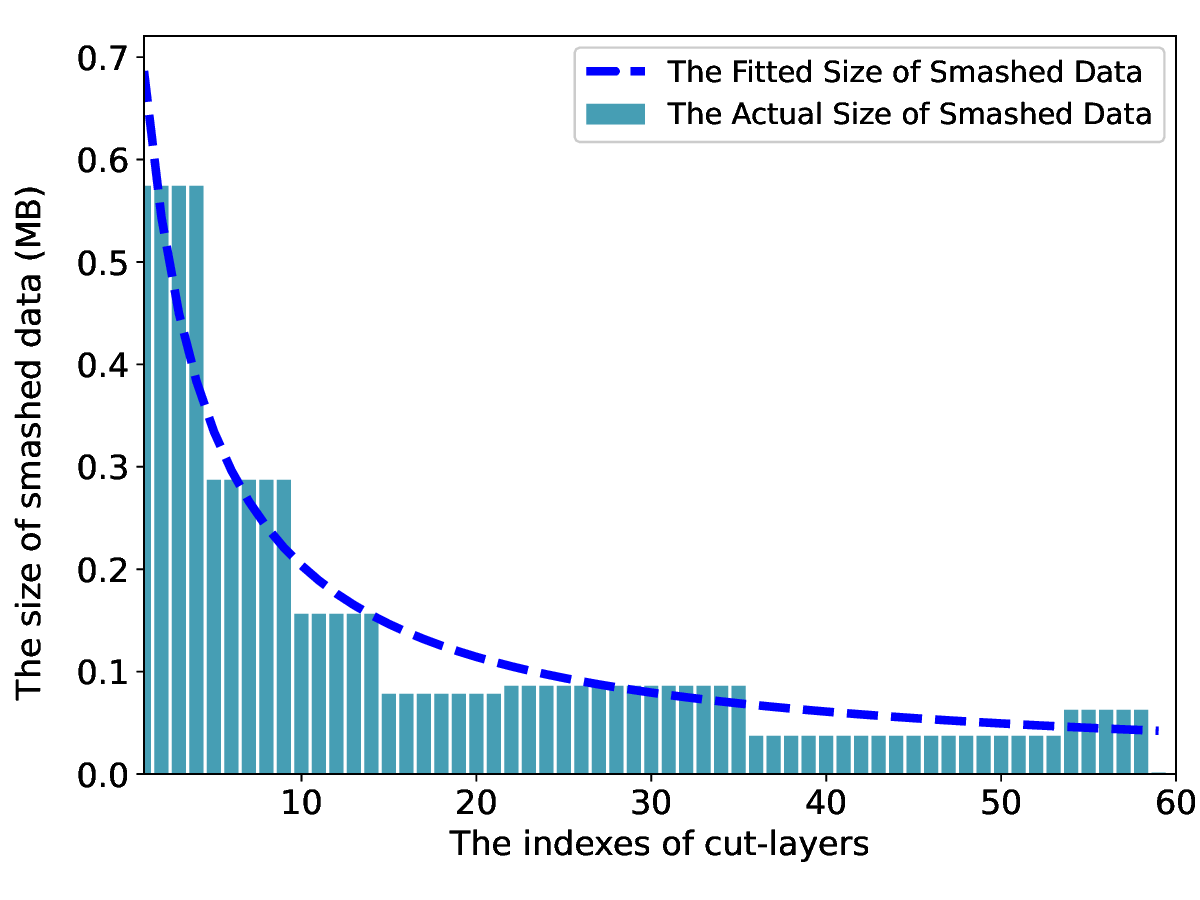}
    \caption{The quantitative-relationship between $l_k$ and $\Lambda_k$.}
    \label{fig:Smashed output size fitting effect}
\end{figure}

\subsection{Problem Reformulation}
By substituting eqs. \eqref{model size fitting function}, \eqref{kapa}, and \eqref{smashed-data size fitting function} into eq. \eqref{original expression for T_k}, the expression for $T_k$ can be written as
\begin{equation}
     T_k =  
    \begin{cases} 
        2\frac{\alpha l_k^2}{r_k} + I_k \left|\mathcal{B}_k\right| \left( \beta(1+\kappa)\left(\frac{ l_k }{F_k^\text{C}}+ \frac{  L - l_k  }{F_k^\text{S}} \right) + 2\frac{\frac{\gamma_1}{l_k+\gamma_2}}{r_k} \right), \\ 
        \quad \quad \quad \quad \quad \quad \quad \quad \quad \quad\quad \text{if}\ l_k^\text{min} \leqslant l_k < L,\ \forall k \in \mathcal{K},\   \\
        \displaystyle {\frac{2|\textbf{w}|}{r_k}+ I_k |\mathcal{B}_k| \frac{\Gamma}{f^\text{C}_k}},\ \text{if}\ l_k=L,\ \forall k \in \mathcal{K}.\\
    \end{cases}
    \label{fitted expression for T_k}
\end{equation}
Then, the MINLP problem \eqref{minT_total_original:mainfun mathcal T by lf} is transformed into the following continuous problem. 

\begin{align}
    & \min_{\textbf{L},\textbf{F}} \mathcal{T}, \; \;  \mathcal{T}=\max_{k\in\mathcal{K}}\ T_k,  \label{minT_total: continuoous by lfsfc} \\
    \mbox{s.t.}\;\;\;
    & l_k^\text{min} \leqslant l_k \leqslant L,\ \forall k \in \mathcal{K}, \tag{\ref{minT_total: continuoous by lfsfc}.1} \label{minT_total:subject l_m range} \\
    & \sum^{K}_{k=1} f^\text{S}_k \leqslant F^{\text{max}}.  \tag{\ref{minT_total: continuoous by lfsfc}.2} \label{minT_total:subject_f_max}
\end{align}

One can first find the solution of problem \eqref{minT_total: continuoous by lfsfc}, and then, rounds it to an integer.
We note that if the computing resource allocation $\textbf{F}$ for the $K$ clients were known, problem \eqref{minT_total: continuoous by lfsfc} can be decomposed into $K$ independent subproblems, and the goal of the $k^\text{th}$ subproblem is to minimize the latency for client $k$ to complete a local training session. These subproblems can be solved in parallel without loss of optimality.

Based on the above findings, we can solve the \textit{cut-layer selection} problem (to find the optimal \textbf{L}) and the \textit{computing resource allocation} problem (to find the optimal $\textbf{F}$) alternately and iteratively. Finally, the solution to problem \eqref{minT_total: continuoous by lfsfc} can be obtained.

\subsection{Solving The Cut-Layer Selection Problem}
For any given $\textbf{F}$, problem \eqref{minT_total: continuoous by lfsfc} can be decomped into $K$ independent subproblems as given below. 
\begin{alignat}{1}
    \min_{l_k} & \; T_k,\  \forall k \in \mathcal{K}, \label{minimize mathcal T using L} \\
    \mbox{s.t.}\;\;\; 
    & l_k^\text{min} \leqslant l_k \leqslant L,\ \forall k \in \mathcal{K}. \tag{\ref{minimize mathcal T using L}.1} 
\end{alignat}

By solving problem \eqref{minimize mathcal T using L}, the optimal \textit{cut-layer} of any client $k$ is obtained in closed-form, as shown in the following lemma.

\begin{lemma}
    With $\alpha > 0,\ \beta > 0,\ \kappa>0,\ \gamma_1 > 0,\ and\ \gamma_2 \geqslant 0$, the optimal cut-layer $l^*_k$ for any client $k$ is given by the following rules.
    \begin{align} 
        l^*_k = & 
        \begin{cases} 
            l_k^\text{min}, & \frac{\partial T_{k}}{\partial l_k} \big|_{l_k = l_k^\text{min}} > 0,  \\
            L, & \frac{\partial T_{k}}{\partial l_k} \big|_{l_k = L} < 0,\\ 
            \left \lfloor  \underset{l_k}{\operatorname{arg}}\, \left(\frac{\partial T_{k}}{\partial l_k}=0 \right) \right \rfloor , & \frac{\partial T_{k}}{\partial l_k} \big|_{l_k = l_k^\text{min}} \leqslant 0\ and\ \frac{\partial T_{k}}{\partial l_k} \big|_{l_k = L} \geqslant 0. 
        \end{cases} \; \label{The optimal splitting layer of k}
    \end{align}
    \label{lemma: optimal cut-layer}
\end{lemma}
\begin{proof}
    Please refer to Appendix A.
\end{proof}

In the $3^\text{rd}$ case of eq. \eqref{The optimal splitting layer of k}, the function $\frac{\partial T_k}{\partial l_k}=0$ contains a cubic terms of $l_k$ and can be solved by using the \textit{Cardano’s formula} \cite{9408263}. It is worth noting that the complexity of finding the suboptimal \textit{cut-layers} for the $K$ clients is only $\mathcal{O}(K)$, thus avoiding an exhaustive search on the actual complex relationships between the parameters of an AI-model with extremely high time and space complexity.

\subsection{Computing Resource Allocation of the PS}
Substituting the obtained \textit{cut-layers} of the $K$ clients, namely, $\textbf{L}^*=(l_1^*, \cdots, l_K^*)$, into eq. \eqref{fitted expression for T_k}, one can get
\begin{equation}
     T_k =  
    \begin{cases} 
        2\frac{\alpha (l_k^*)^2}{r_k} + I_k \left|\mathcal{B}_k\right| \left( \beta(1+\kappa)\left(\frac{ l_k^* }{F_k^\text{C}}+ \frac{  L - l_k^*  }{F_k^\text{S}} \right) + 2\frac{\frac{\gamma_1}{l_k^*+\gamma_2}}{r_k} \right), \\ 
        \quad \quad \quad \quad \quad \quad \quad \quad \quad \quad\quad \text{if}\ l_k^\text{min} \leqslant l_k^* < L,\ \forall k \in \mathcal{K},\   \\
        \displaystyle {\frac{2|\textbf{w}|}{r_k}+ I_k |\mathcal{B}_k| \frac{\Gamma}{f^\text{C}_k}},\ \text{if}\ l_k^*=L,\ \forall k \in \mathcal{K}.\\
    \end{cases} 
\end{equation}
Problem \eqref{minT_total: continuoous by lfsfc} can be simplified to 
\begin{alignat}{1}
    \min_{\textbf{F}} \;\; & \mathcal{T}=\max_{k\in\mathcal{K}}\ T_k, \label{minimize mathcal T using F} \\
    \mbox{s.t.}\;\;\; 
    & \sum^{K}_{k=1} f^\text{S}_k \leqslant F^{\text{max}}. \tag{\ref{minimize mathcal T using F}.1} \label{minimize mathcal T using F:sum of PS frequency}    
\end{alignat}

The difficulty in solving problem \eqref{minimize mathcal T using F} is that the objective, $\mathcal{T}=\max_{k\in\mathcal{K}}T_k$, to be optimized is a nonlinear function w.r.t $T_k$, and $T_k$ is a piecewise function w.r.t $l_k^*$. To solve this difficulty, we define $\tilde{T}_k$ as the latency for client $k$ to complete a local training session using only the local computing power $f_k^\text{C}$.
We can sort the $K$ clients in ascending order according to $\{\tilde{T}_k|\forall k \in\mathcal{K}\}$ as shown in eq. \eqref{new index}, where the reordered index of client $k$ is represented as $\tilde{k}$, $\forall \tilde{k} \in \mathcal{K}$.
\begin{align}
    \overbrace{T_{\tilde{1}} \leqslant \cdots \leqslant T_{\tilde{k}-1=\theta-1}}^\texttt{FedAvg} \leqslant \overbrace{T_{\tilde{k}=\theta} \leqslant \cdots \leqslant T_{K}}^\texttt{SFL} \label{new index}
\end{align}
Next, we derive the following two lemmas.

\begin{lemma}
    The $K$ clients participating in an FL task can be divided into two groups: the clients in group $\mathcal{K}_\texttt{FedAvg} = \{\tilde{1}, \cdots, \theta-1\}$ adopt $\texttt{FedAvg}$ and train $\textbf{w}_{\tilde{k}}$ independently, while the clients in group $\mathcal{K}_\text{\texttt{SFL}} = \{\theta, \cdots, \tilde{K} \}$ adopt \texttt{SFL} and train $\textbf{w}_{\tilde{k}}$ in collaboration with the PS, where $\theta$ is the index of the first client in $\mathcal{K}_\text{\texttt{SFL}}$.  \label{lemma 1}
\end{lemma}
\begin{proof}
    Please refer to Appendix B.
\end{proof}

\begin{lemma}
    The optimum of the objective of problem \eqref{minimize mathcal T using F} is $\mathcal{T} = T_{\theta}$, $\theta \in \tilde{\mathcal{K}}$. \label{group division}   
\end{lemma}
\begin{proof}
    Please refer to Appendix C.    
\end{proof}   
 
By using \textbf{Lemmas \ref{lemma 1}} and \textbf{\ref{group division}}, problem \eqref{minimize mathcal T using F} can be converted into the following form. 
\begin{alignat}{1}
    \min_{\theta,\ \textbf{F}} \;\; & T_{\theta}, \label{subproblem 3} \\
    \mbox{s.t.}\;\;\; 
    & \mathcal{K}_\texttt{FedAvg} \cup \mathcal{K}_\text{\texttt{SFL}} = \mathcal{K}, \tag{\ref{subproblem 3}.1}\\
    & \mathcal{K}_\texttt{FedAvg} \cap \mathcal{K}_\text{\texttt{SFL}} = \emptyset, \tag{\ref{subproblem 3}.2}\\
    & T_{\tilde{k}} \leqslant T_{\theta},\ \forall \tilde{k} \in \mathcal{K}_{\texttt{FedAvg}}, \tag{\ref{subproblem 3}.3}\\
    & T_{\tilde{k}} =T_{\theta},\ \forall \tilde{k} \in \mathcal{K}_\text{\texttt{SFL}}. \tag{\ref{subproblem 3}.4}
\end{alignat}

By solving problem \eqref{subproblem 3}, the optimal amount of computing resources that the PS allocates to any client $k$ is obtained in closed-form, as shown in the following lemma.
\begin{lemma}
    The amount of computing resources allocated by the PS to each client $\tilde{k}\in\mathcal{K}$ is given by the following formula.
    \begin{align}
        f^\text{S}_{\tilde{k}} = 
        \begin{cases} 
            0, & \forall \tilde{k} \in \mathcal{K}_{\texttt{FedAvg}}, \\ 
            \frac{I_{\tilde{k}} \left|\mathcal{B}_{\tilde{k}}\right| (F^\text{S}_{\tilde{k}} + B^\text{S}_{\tilde{k}}) }{ T_\theta - \frac{ I_{\tilde{k}} \left|\mathcal{B}_{\tilde{k}}\right| (F^\text{C}_{\tilde{k}} + B^\text{C}_{\tilde{k}}) }{ f^\text{C}_{\tilde{k}} } - 2 \frac{I_{\tilde{k}} \left|\mathcal{B}_{\tilde{k}}\right| \Lambda_{\tilde{k}} + |\textbf{w}^\text{C}_{\tilde{k}}|}{r_{{\tilde{k}}}} } , & \forall \tilde{k} \in \mathcal{K}_{\texttt{SFL}}. \\ 
        \end{cases} \label{f_S_k representation using mathcal T}
    \end{align} 
    \label{optimal solution}
\end{lemma}

\begin{proof}
    Please refer to Appendix D.
\end{proof}

According to \textbf{Lemma \ref{optimal solution}}, once $\mathcal{T}=T_\theta$ is known, the optimal resource allocation of the PS, $\textbf{F}$, can be obtained. In order to solve $T_\theta$, we substitute formula \eqref{f_S_k representation using mathcal T} into constraint \eqref{minimize mathcal T using F:sum of PS frequency} and convert this inequality constraint into an equality constraint as  
\begin{align}
    & \sum_{\tilde{k} \in \mathcal{K}_{\texttt{SFL}} } \frac{I_{\tilde{k}} \left|\mathcal{B}_{\tilde{k}}\right| (F^\text{S}_{\tilde{k}} + B^\text{S}_{\tilde{k}}) }{ T_\theta - \frac{ I_{\tilde{k}} \left|\mathcal{B}_{\tilde{k}}\right| (F^\text{C}_{\tilde{k}} + B^\text{C}_{\tilde{k}}) }{ f^\text{C}_{\tilde{k}} } - 2 \frac{I_{\tilde{k}} \left|\mathcal{B}_{\tilde{k}}\right| \Lambda_{\tilde{k}} + |\textbf{w}^\text{C}_{\tilde{k}}|}{r_{\tilde{k}}} } = F^\text{max}
    \label{solve t_theta}
\end{align}

\begin{lemma}
    $T_\theta$ is strictly convex w.r.t $F^\text{max}$.
    \label{lemma 3}
\end{lemma}
\begin{proof}
    Please refer to appendix E.    
\end{proof}

\textbf{Lemma \ref{lemma 3}} ensures that for a given $F^\text{max}$ an unique solution $T_\theta$ exists for eq. \eqref{solve t_theta}, which can be solved by using the CVX tool \cite{grant2013cvx}.

\subsection{Algorithm Implementation and the Convergence and Complexity Analysis}
The overall algorithm to solve problem \eqref{minT_total_original:mainfun mathcal T by lf} is summarized in the following Algorithm \ref{Optimization of problem with one iteration}.

\begin{algorithm}
\caption{The overall algorithm to solve problem \eqref{minT_total_original:mainfun mathcal T by lf}. }
\label{Optimization of problem with one iteration}
\begin{small}
\begin{algorithmic}[1]
    
    \STATE Use eq. \eqref{Time consumption T_k at layer L} to obtain $\tilde{T}_k$, $\forall k \in \mathcal{K}$;
    \STATE Sort the $K$ clients in ascending order according to $\{\tilde{T}_k|k \in\mathcal{K}\}$, and update the index of client $k$ to its serial number $\tilde{k}$;
    \STATE Set the iteration number $i = 0$, $\textbf{F}^{i} = \{\frac{F^\text{max}}{K}, \cdots, \frac{F^\text{max}}{K}\}$, $\theta ^i= 1$, and $\mathcal{T}=T_{\theta^i}$;

    \REPEAT
        \STATE With the given $\textbf{F}^i$, select \textit{cut-layers} $\textbf{L}^i$ for the clients using eq. \eqref{The optimal splitting layer of k};
        
        \STATE Substitute $\textbf{L}^i$ and $\theta ^i$ into eq. \eqref{subproblem 3} and obtain $T_{\theta^{i+1}}$ using the CVX tool \cite{grant2013cvx}; 
        
        \STATE With the obtained $T_{\theta^{i+1}}$, update $\textbf{F}^{i+1}$ using eq. \eqref{f_S_k representation using mathcal T};
        
        \STATE $\theta^{i+1} = \Phi \left(f^\text{S}_{\tilde{1}},  \cdots, f^\text{S}_{\tilde{K}}\right)$;

        \STATE $i=i+1$;

    \UNTIL $\mathcal{T}=T_{\theta^i}$ converges, or the iteration number reaches the maximum $i=i^\text{max}$;

    \RETURN The computing allocation strategy $\textbf{F}$ of the PS and the cut-layers selected for the $K$ clients $\textbf{L}$.
\end{algorithmic}
\end{small}
\end{algorithm}

In \textit{Step 8}, function $\Phi(\cdot)$ counts the number of non-zero elements in a vector. Since the algorithm contains iterative execution, the convergence is proved below.

% Convergence analysis.
\begin{lemma}
    Algorithm \ref{Optimization of problem with one iteration} converges to a stable point.
\end{lemma}

\begin{proof}
    For a feasible solution $\left(\textbf{L}, \textbf{F}, \theta \right)$ of problem \eqref{minT_total_original:mainfun mathcal T by lf}, the resulting objective value can be obtained as $\mathcal{T} (\textbf{L},\textbf{F},\theta)$.
    
    In the $i^\text{th}$ iteration, the algorithm ensures that 
    \begin{align}
        \mathcal{T} \left(\textbf{L}^i,\textbf{F}^i,\theta^i\right) & \geqslant \mathcal{T} (\textbf{L}^{i+1},\textbf{F}^i,\theta^i) & Step\ 5 & \label{alternating minimization step 1 inequality}
        \\
        & \geqslant \mathcal{T} (\textbf{L}^{i+1},\textbf{F}^{i+1},\theta^{i}) & Step\ 6 & \label{alternating minimization step 2 inequality} \notag
        \\
        & \geqslant \mathcal{T} (\textbf{L}^{i+1},\textbf{F}^{i+1},\theta^{i+1}) & Step\ 7 & . \notag
    \end{align}
    
    Therefore, $\left\{\mathcal{T} (\textbf{L}^i,\textbf{F}^i,\theta^i)| i=1, \cdots, i^{\text{max}} \right\}$, the sequence of the objective values of problem \eqref{minT_total_original:mainfun mathcal T by lf}, is non-increasing.
    Furthermore, because $\mathcal{T}>0$, the alternate-optimization-based algorithm can converge to a lower bound as the number of iterations increases.
\end{proof}

Finally, the computational complexity of \textbf{Algorithm \ref{Optimization of problem with one iteration}} is analyzed as follows.
The complexity of ranking $K$ clients in \textit{Step 2} is $\mathcal{O}(K^2)$. The complexity of performing \textit{cut-layer} selection for the $K$ clients in \textit{Step 5} is $\mathcal{O}(K)$.
In \textit{Step 6}, the optimal $T_\theta$ could be obtained by using linear search which requires $\mathcal{O}(K^3)$ to converge \cite{Boyd2004ConvexO}. 
The resource allocation of the PS in \textit{Step 7} requires $\mathcal{O}(K)$.
As a result, the computational complexity of \textbf{Algorithm \ref{Optimization of problem with one iteration}} is $\mathcal{O}\left(K^2+ i^\text{max}(K^3 + 2K)\right)$.

Although the alternate-optimization-based \textbf{Algorithm \ref{Optimization of problem with one iteration}} is with polynomial time complexity, it generally converges at a sub-optimal solution to the original MINLP problem, namely, problem (\ref{minT_total_original:mainfun mathcal T by lf}), due to the \textit{regression method} employed. The high quality of the solution will be validated in the following experiments.

\section{Experiment Results}
To verify the performance of the proposed \texttt{SFL} method, we experiment on the \textit{EfficientNetV2} \cite{tan2021efficientnetv2}, a popular DNN-model with $L=59$ layers. The proposed \texttt{SFL} method is also applicable to other AI-models with different neural network structures.
The \textit{EfficientNetV2} is trained to perform image classification task on the dataset MNIST \cite{deng2012mnist}, which has 60,000 labeled samples (each sample consists of 28$\times$28 pixels) in 10 classes with 6,000 images per class.
The samples of each class is evenly distributed across the clients involved in an FL task.
By using the data generation method \cite{9261995}, the number of samples of each client ranges from 361 to 3,578, and the local dataset is split randomly with 75\% for training and 25\% for testing.
The SGD algorithm with the fixed batch size $|\mathcal{B}_k|=32$ is adopted to train the local model $\textbf{w}_k$ of any client $k$.

The available computing resources of the PS is set to $F^\text{max}=3,000$ G\textit{Flops}, and a total of 30 clients serve as the candidates for the PS-initiated FL tasks. The clients are heterogeneous and have different local computing power $f_k^\text{C}$ or transmission rates $r_k$.
The PS randomly selects a fraction of the candidate clients to participate in each round of global training.

\subsection{Effectiveness of The Regression Method in Finding The Optimal Cut-Layer}
\label{chapter 5.A}
The \textit{regression method}, proposed in Sec. V-A, is used to help a client quickly find the approximately optimal \textit{cut-layer}.
To verify its validity, we set the local computing power of client $k$ to $f_k^\text{C}=100$ G\textit{Flops}, the edge computing power of the PS allocated to client $k$ to $f_k^\text{S}=1,484$ G\textit{Flops}, the transmission rate to $r_k=4$ Mb/s, and the local dataset size of client $k$ to $|\mathcal{D}_k|=3,396$.
The latency for client $k$ to complete a local training session with different \textit{cut-layers} is shown in Fig. \ref{FL compare with SFL}. For each \textit{cut-layer}, the latency of training the \textit{client-side} model $\textbf{w}^{\text{C}}_k$ and the \textit{server-side} model $\textbf{w}^{\text{S}}_k$ and transmitting the \textit{intermediate results} ($\mathcal{S}_{k,n}$ and $\mathcal{G}_{k,n}$) are separately annotated in Fig. \ref{FL compare with SFL}. 

% Global_training_time_VS_cut_layer
\begin{figure}[htbp]
    \centering
    \includegraphics[width=8.8cm,height=5.5cm]{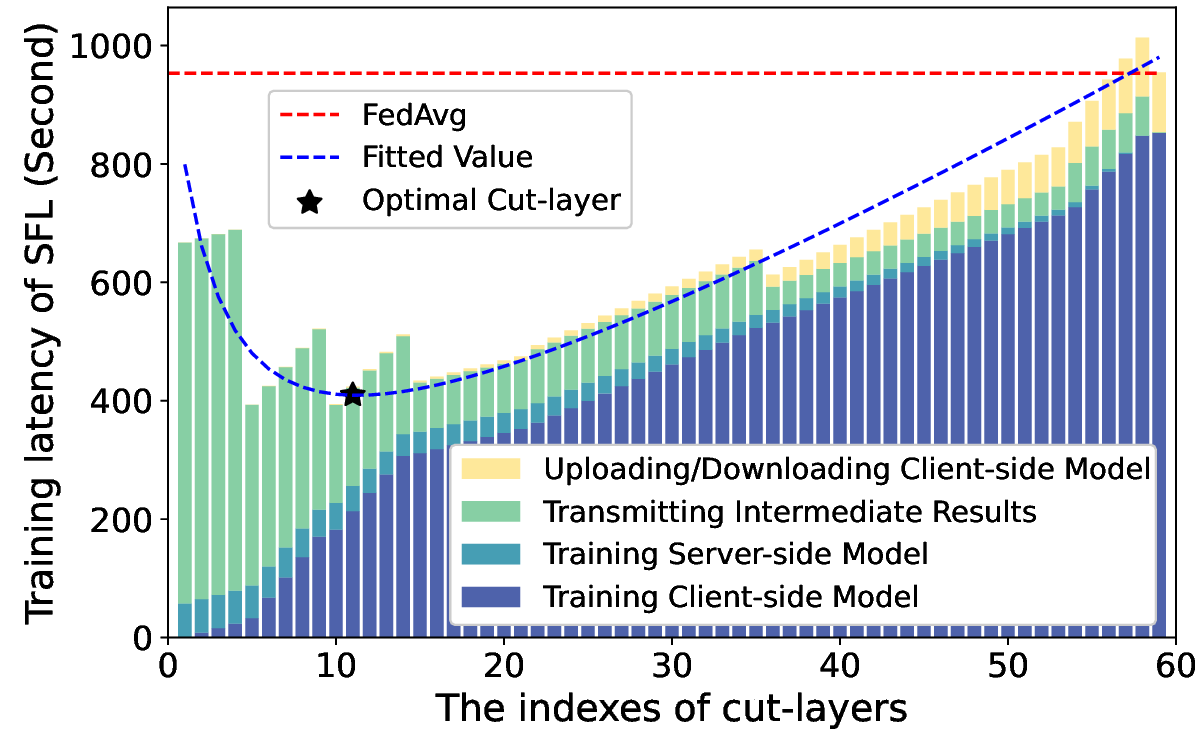}
    \caption{Training latency of \texttt{SFL} with different \textit{cut-layers}.}
    \label{FL compare with SFL}
\end{figure}

As shown in Fig. \ref{FL compare with SFL}, the \texttt{SFL} can considerably reduce the training latency of client $k$, compared to the \texttt{FedAvg}. Specially, when the optimal \textit{cut-layer} $l_k=11$ is selected, up to 300 seconds can be saved.
When $l_k < 11$, the training latency mainly results from transferring the \textit{intermediate results} between the client and PS. When $l_k > 11$, the training latency mainly comes from the training of the \textit{client-side} model $\textbf{w}^{\text{C}}_k$. In all cases, the latency for the PS to train the \textit{server-side} model $\textbf{w}^{\text{S}}_k$ and the latency for the client to download/upload the \textit{client-side} model $\textbf{w}^{\text{C}}_k$ are negligible for the overall training latency.

From Fig. \ref{FL compare with SFL}, we also note that the optimal \textit{cut-layer} for the client found by using the \textit{regression method} is consistent with that obtained by an exhaustive search on the space created by the actual quantitative relationship between the parameters of \textit{EfficientNetV2}. In this way, the high computational complexity caused by searching for the optimal \textit{cut-layer} of an AI-model is avoided.

To quantify the prediction accuracy of the proposed \textit{regression method}, we adopt the determination coefficient $\mathcal{R}$ \cite{9812907} as the evaluation metric. $\mathcal{R} \triangleq 1 - \frac{\sum^{L}_{l=1} (x - \hat{x}_{l})^2 }{ \sum^{L}_{l=1} (x_l - \tilde{x}_l)^2 }$, where $x_l$, $\hat{x}_l$, and $\tilde{x}_l$ represent the true value, the predicted value, and the historical mean value, respectively.
The closer the value of $\mathcal{R}$ is to 1, the better the regression performance is.
In the following Tab. \ref{tab:Performances of Curve Fitting}, we show the determination coefficient $\mathcal{R}$ of the fitted values of $|\textbf{w}^\text{C}_k|$ (see eq. \eqref{model size fitting function}), $F_k^\text{tot}$ (see eq. \eqref{kapa}) and $\Lambda_k$ (see eq. \eqref{smashed-data size fitting function}).

\begin{table}[!htbp]
\newcommand{\tabincell}[2]{\begin{tabular}{@{}#1@{}}#2\end{tabular}} %放在导言区,表格内换行    
\begin{center}
    \caption{The determination coefficient of the proposed regression method.}
    \label{tab:Performances of Curve Fitting}
    \begin{tabular}{R{2.5cm}C{0.6cm}L{1.5cm}}
        \toprule  %添加表格头部粗线
        Fitted Item && $\mathcal{R}$ \\
        \midrule  %添加表格中横线

        $|\textbf{w}^\text{C}_k|$ && 0.9482 \\
        $F_k^\text{tot}$ && 0.9659 \\
        $\Lambda_k $ && 0.9065 \\
        
        \bottomrule %添加表格底部粗线
    \end{tabular}
\end{center}
\end{table}

As shown in Tab. \ref{tab:Performances of Curve Fitting}, $0.9<\mathcal{R}<1$ for all the predicted items. This indicates that the proposed \textit{regression method} can accurately predict the relationship between the selected \textit{cut-layer} and other parameters of an AI-model, which ensures that the proposed alternate-optimization-based \textbf{Algorithm \ref{Optimization of problem with one iteration}} can obtain high-quality solutions as analyzed in Sec. V-E.

\subsection{Results for joint cut-layer selection and computing resource allocation.}
To validate the performance of the proposed joint \textit{cut-layer} selection and computing resource allocation algorithm, namely \textbf{Algorithm \ref{Optimization of problem with one iteration}}, we select 10 clients from the 30 candidates, each with the same number of local training epochs $I_k=20$.
Fig. \ref{Frequency allocation effect} compares the training latency for each client using the proposed \texttt{SFL} to the vanilla \texttt{FedAvg}. For the \texttt{SFL}, the latency for each client to train the \textit{client-side} and \textit{server-side} models and communicate the \textit{intermediate results} are also separately annotated in Fig. \ref{Frequency allocation effect}.

\begin{figure}[htbp]
    \centering
    \includegraphics[width=8.8cm,height=6.5cm]{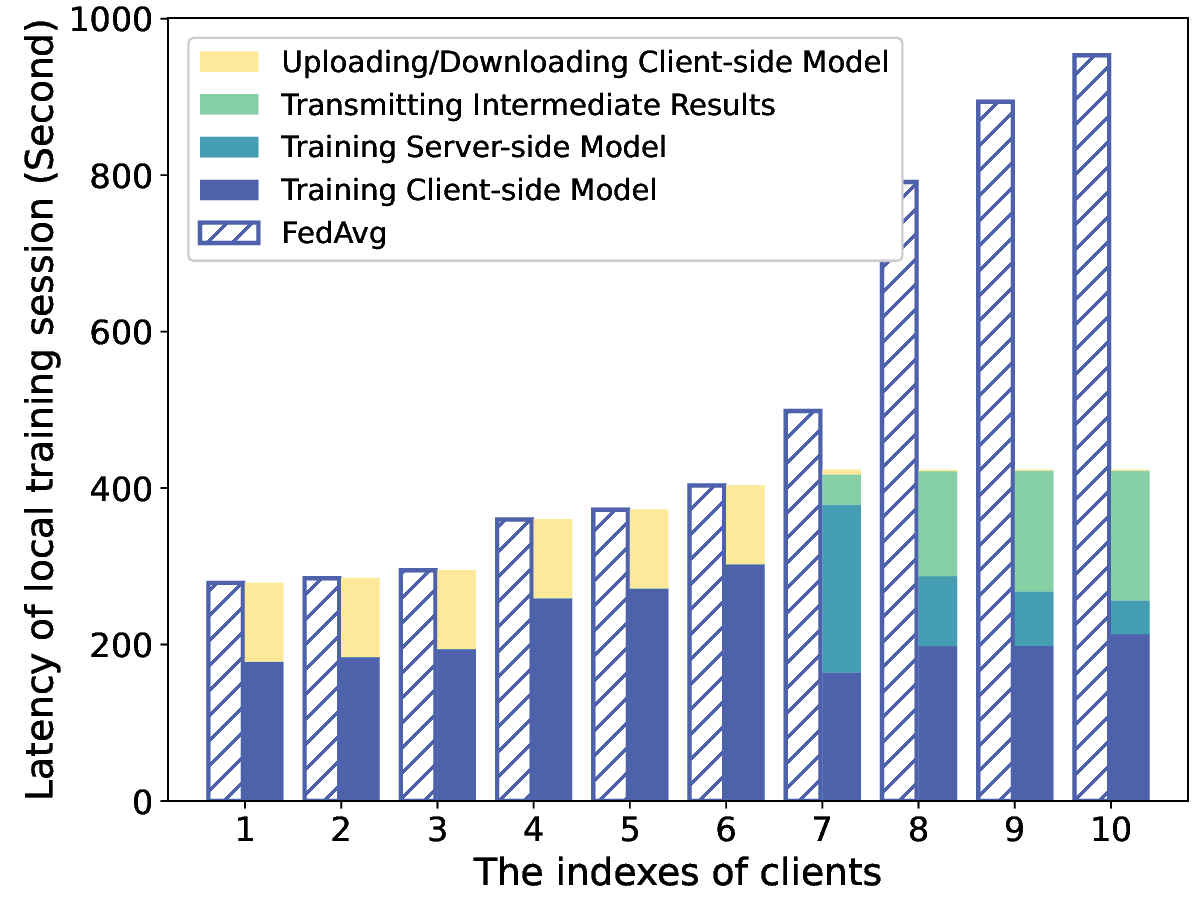}
    \caption{Training latency of different clients in the \texttt{SFL}.}
    \label{Frequency allocation effect}
\end{figure}

As can be seen from Fig. \ref{Frequency allocation effect}, the PS did not allocate any resources to clients 1 to 6 in the \texttt{SFL}, so these clients used the \texttt{FedAvg} to train the local model, while clients 7 to 10 shared the computing resources of the PS to reduce the training latency.
As a result, clients 7 to 10 have roughly the same training latency (410 s), but the training latency of clients 1 to 6 is different and each is lower than that of clients 7 to 10.
This verifies that the proposed \texttt{SFL} has the similar function as the \textit{water pouring algorithm} \cite{6223460}, which can allocate the computing resources of the PS to the resource-constraint clients adaptively, thus reducing the difference in training latency of the clients.
In contrast, when all the clients adopt \texttt{FedAvg}, they can only train the local model $\textbf{w}_k$ independently. As the overall training latency depends on the longest training time of the clients, the latency of the \texttt{FedAvg} (980 s) is much higher than that of the \texttt{SFL} (410 s). 

Next, we demonstrate the convergence of \textbf{Algorithm \ref{Optimization of problem with one iteration}}. As the number of iterations increases, the convergence of the \textit{cut-layer} selection and computing resource allocation are shown in Fig. \ref{Cut-layer Optimal Process} and Fig. \ref{Frequency Allocation Optimal Process}, respectively, and the resulting training latency is shown in Fig. \ref{Total Time Consumption Optimal Process}. It can be seen that \textbf{Algorithm \ref{Optimization of problem with one iteration}} can converge to a fixed point after about 5 iterations. This verifies the low computational complexity of the proposed algorithm.

\begin{figure}[htbp]%
    \centering
    \subfloat[Convergence of \textit{cut-layer} selection.]{
        \label{Cut-layer Optimal Process}
        \includegraphics[width=0.23\textwidth]{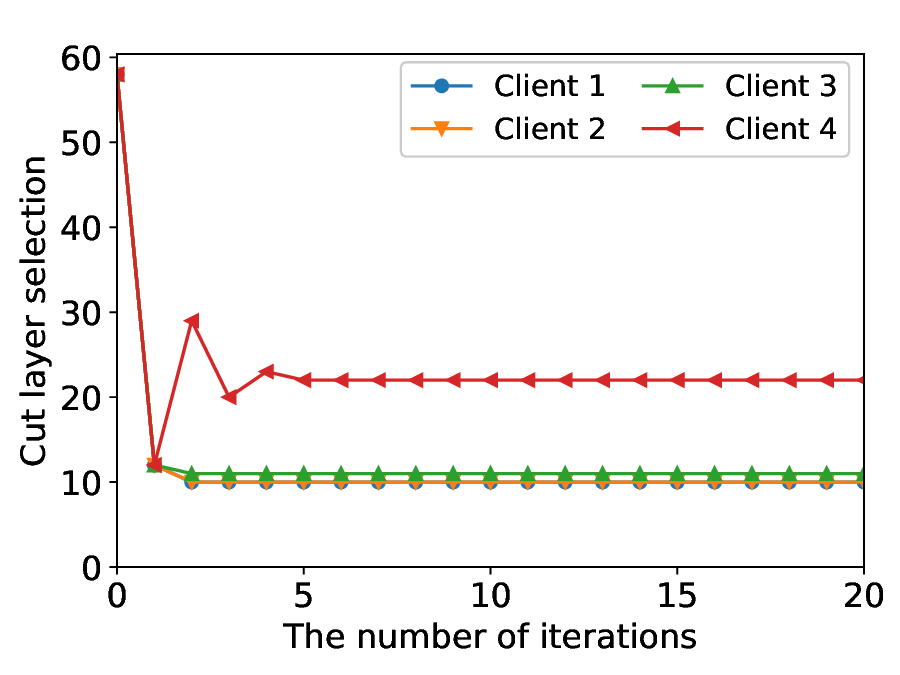}
    }\hfill
    \subfloat[Convergence of resource allocation.]{
        \label{Frequency Allocation Optimal Process}
        \includegraphics[width=0.23\textwidth]{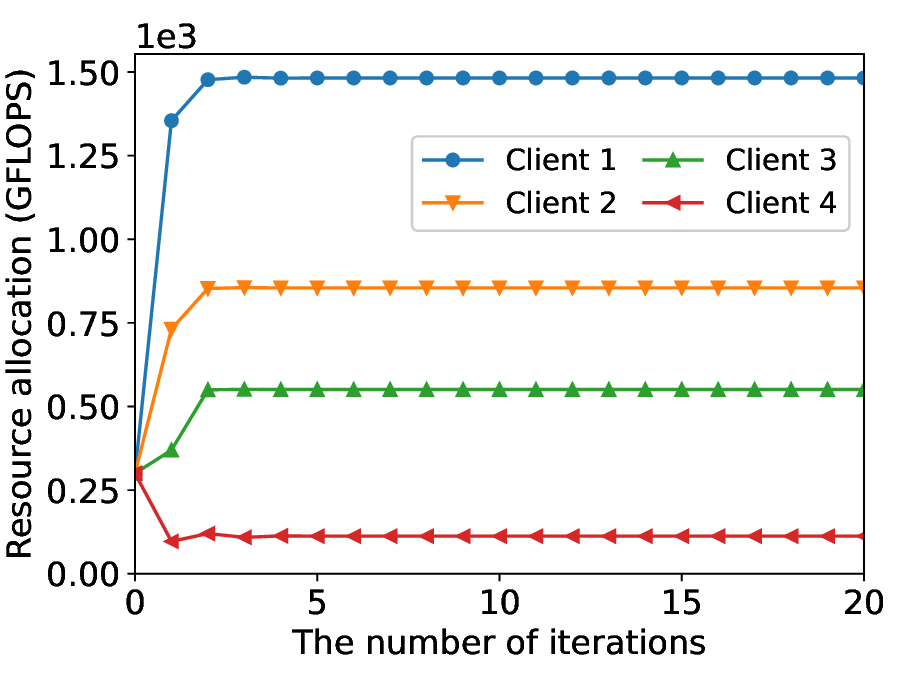}
    }
    \caption{Convergence of Algorithm \ref{Optimization of problem with one iteration} for the two subproblems.}
    \label{Cut Layer and Frequency Allocation Optimal Process}
\end{figure}

\begin{figure}[htbp]
    \centering
    \includegraphics[width=4cm,height=3cm]{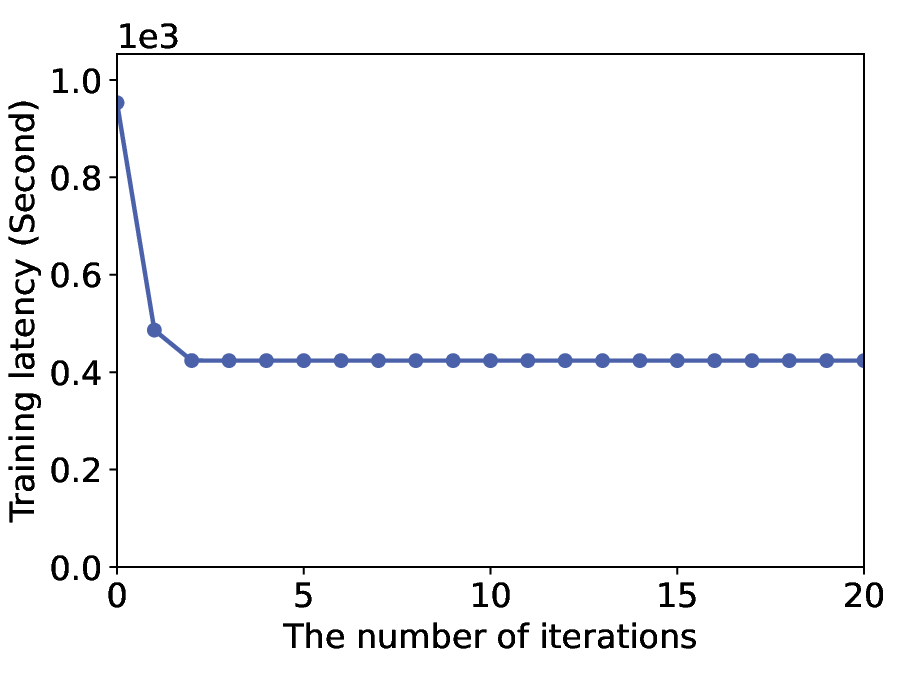}
    \caption{Convergence of the training latency of \texttt{SFL}.}
    \label{Total Time Consumption Optimal Process}
\end{figure}

Finally, we show the effect of the amount of computing resources available for the PS on the overall training latency of the \texttt{SFL}. For that purpose, we increase the computing resources of the PS from $F^\text{max}=100$ G\textit{Flops} to $F^\text{max}=9,000$ G\textit{Flops}. The variation of training latency of the \texttt{SFL} (for one local training session) is shown in Fig. \ref{Influence of total frequency}.

\begin{figure}[htbp]
    \centering
    \includegraphics[width=8.8cm,height=6cm]{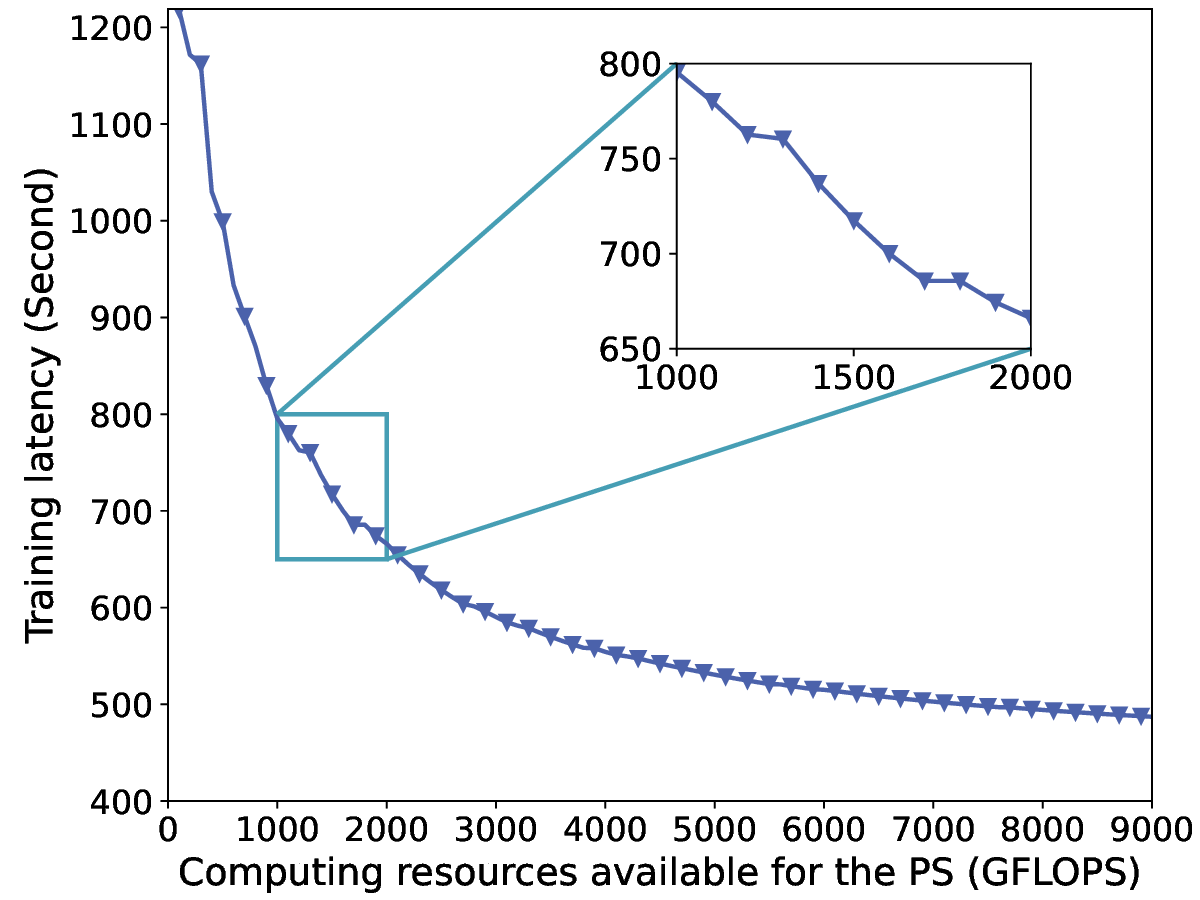}
    \caption{Training latency of \texttt{SFL} with different $F^\text{max}$.}
    \label{Influence of total frequency}
\end{figure}

As shown in Fig. \ref{Influence of total frequency}, with the increasing computing resources available for the PS, the training latency of the \texttt{SFL} gradually decreases. This is because more and more resource-constrained clients can share the computing resources of the PS, so that the system bottleneck (the longest training latency for the clients to complete a local training session) is continuously eliminated, thus reducing the overall training latency of the \texttt{SFL}.

It is also noted from Fig. \ref{Influence of total frequency} that the training latency of the \texttt{SFL} decreases rapidly when the server resources increase from $F^\text{max}=100$ G\textit{Flops} to $F^\text{max}=2,000$ G\textit{Flops}, while the latency decreases slowly when the server resources continue to increase. This indicates that the PS can optimize the amount of resource contributed to an FL task to improve the efficiency of resource utilization. Improving the marginal benefits of server resources is important when the PS simultaneously trains multiple AI-models or performs multi-task learning. This is beyond the scope of this paper and will be left for future studies.

\subsection{Test Accuracy of The Trained Model}
In this part, we show the test accuracy of the trained \textit{EfficientNetV2} after multiple rounds of global training.
In each round, 10 clients were randomly selected from the 30 candidates to participate, each with the same number of local training epochs $I_k=20$.
As the number of global training rounds increases, the time consumed by the \texttt{SFL} and the resulting test accuracy of the trained model are shown in Fig. \ref{SFL accuracy}. Note that each experiment is run at least 3 times and the average results are reported.

\begin{figure}[htbp]
    \centering
    \includegraphics[width=9cm,height=7cm]{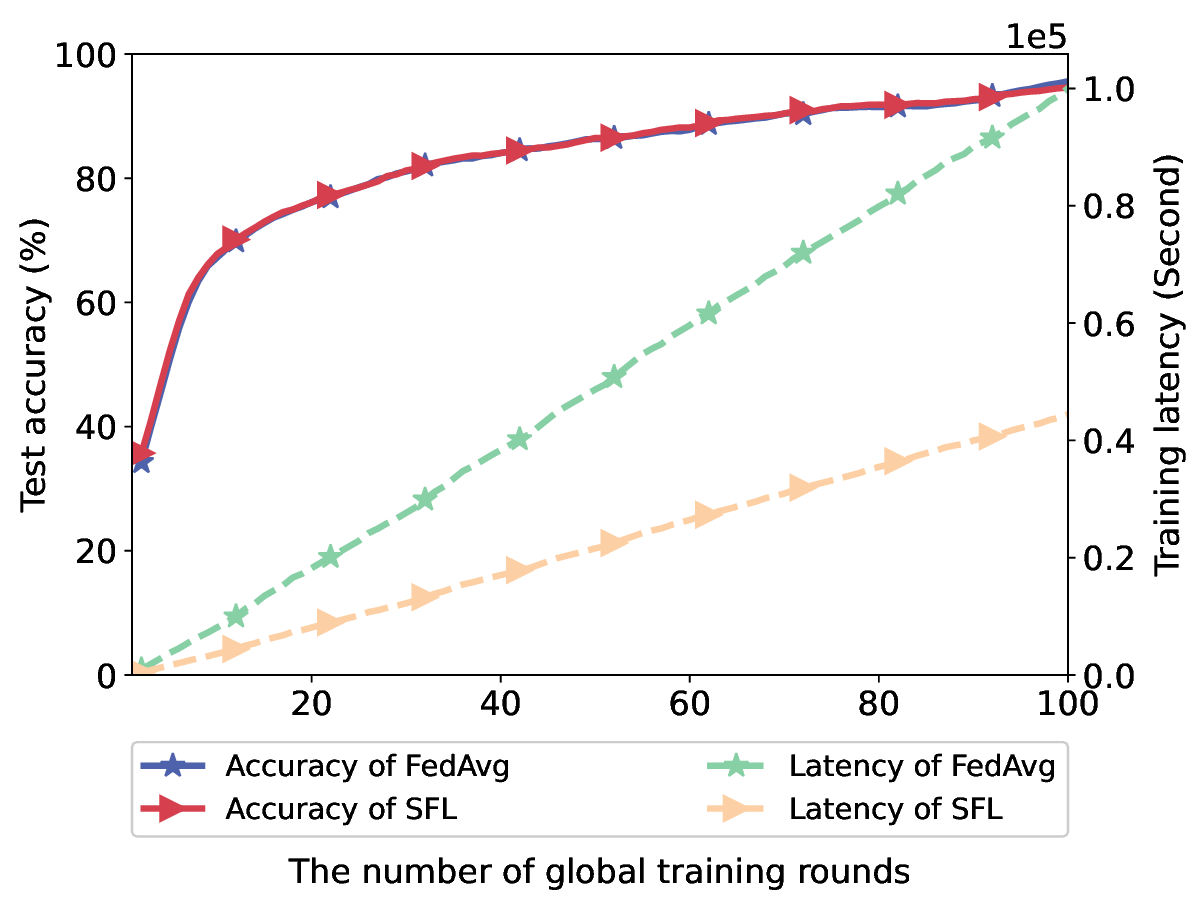}
    \caption{Training latency of the \texttt{SFL} and the resulting test accuracy of \textit{EfficientNetV2}.}
    \label{SFL accuracy}
\end{figure}

From Fig. \ref{SFL accuracy}, one can see that as the number of global training rounds increases, the test accuracies of the models trained with the \texttt{SFL} and \texttt{FedAvg} gradually increase and approach 0.90 after 60 rounds of global training.
Additionally, the accuracy convergence of the \texttt{SFL} is consistent with that of the \texttt{FedAvg}. This verifies the effectiveness of the proposed \texttt{SFL} framework.
But, in order to achieve the same test accuracy, the \texttt{SFL} consumes much less time than the \texttt{FedAvg}. 
This verifies that the proposed \texttt{SFL} can greatly improve the training efficiency of FL without compromising the test accuracy.

\section{Conclusions}
In this paper, we leverage the \textit{edge computing} and \textit{split learning} techniques to improve the training efficiency of \texttt{SFL}. Specially, we minimize the training latency of the SGMU-based \texttt{SFL} without loss of the test-accuracy of the trained model. However, the proposed problem is an MINLP problem which is challenging to solve. So we first propose a \textit{regression method} to transform it into a continuous problem, and then, develop an alternate-optimization-based algorithm with polynomial time complexity to solve it. The convergence of the algorithm is proved and the high quality of the solution is experimentally verified. The experiment results also show that the proposed \texttt{SFL} can train an AI-model to have the same test accuracy with much less training-time compared with the \texttt{FedAvg}. In addition, we also note that the PS can further improve the marginal benefits of its computing resources, which is the key to solving parallel-training of multiple AI-models as well as multi-task learning problems. These topics are interesting and deserve further study in the future.

\normalem
\bibliography{./SFLCiteBib}

\begin{appendices}
\section{}
    \label{appendix:proof of lemma: optimal cut-layer}
\textit{Proof of Lemma 1}: For any $l_k^\text{min} \leqslant l_k \leqslant L$, the first- and second-order derivatives of $T_k$ w.r.t $l_k$ are obtained as
\begin{equation}
    \frac{\partial T_k}{\partial l_k} = 4\frac{\alpha l_k}{r_k} + I_k \left|\mathcal{B}_k\right| \left( \beta(1+\kappa)\left(\frac{1}{F_k^\text{C}} - \frac{ 1 }{F_k^\text{S}}\right) - \frac{2\gamma_1}{r_k \left(l_k+\gamma_2\right)^2} \right),
\end{equation} 
and
\begin{equation}
   \frac{\partial^2 T_k}{\partial {l_k}^2} = 4\frac{\alpha}{r_k} +  \frac{4\gamma_1 I_k \left|\mathcal{B}_k\right|}{r_k \left(l_k+\gamma_2\right)^3},
\end{equation}
respectively.

Because $\alpha>0$, $\gamma_1 > 0$, and $\gamma_2 \geqslant 0$, $\frac{\partial^2 T_k}{\partial {l_k}^2}>0$. It means that $\frac{\partial T_k}{\partial l_k}$ increases monotonically w.r.t $l_k$.
The optimal value of $l_k$, denoted by $l^*_k$, can be given by considering the following three cases.

\underline{Case 1}: $\frac{\partial T_{k}}{\partial l_k} \big|_{l_k = l_k^\text{min}} > 0$. If $\frac{\partial T_{k}}{\partial l_k} \big|_{l_k = l_k^\text{min}} > 0$, one can get $\frac{\partial T_k}{\partial l_k}>0$ for any $l_k^\text{min} \leqslant l_k \leqslant L$, because $\frac{\partial^2 T_k}{\partial {l_k}^2}>0$. 
    Thus, $T_k$ increases monotonically w.r.t $l_k$ and reaches the minimum at the left boundary of its domain. Therefore, $l^*_k=l_k^\text{min}$.

\underline{Case 2}: $\frac{\partial T_{k}}{\partial l_k} \big|_{l_k = L} < 0$. If $\frac{\partial T_{k}}{\partial l_k} \big|_{l_k = L} < 0$, we can get $\frac{\partial T_k}{\partial l_k}<0$ for any $l_k^\text{min} \leqslant l_k \leqslant L$, as $\frac{\partial^2 T_k}{\partial {l_k}^2}>0$. 
Thus, $T_k$ decreases monotonically w.r.t $l_k$ and reaches the minimum at the right boundary of its domain $l^*_k=L$.
    
\underline{Case 3}: $\frac{\partial T_{k}}{\partial l_k} \big|_{l_k = l_k^\text{min}} \leqslant 0$ and $\frac{\partial T_{k}}{\partial l_k} \big|_{l_k = L} \geqslant 0$. Because $\frac{\partial^2 T_k}{\partial {l_k}^2}>0$, $\frac{\partial T_{k}}{\partial l_k} $ is monotonic and continuous w.r.t to $l_k$. When $\frac{\partial T_{k}}{\partial l_k} \big|_{l_k = l_k^\text{min}} \leqslant 0$ and $\frac{\partial T_{k}}{\partial l_k} \big|_{l_k = L} \geqslant 0$, there exists a unique solution $l_k$ for $\frac{\partial T_{k}}{\partial l_k} = 0$ within $l_k \in [l_k^\text{min}, L]$. By rounding the resulting $l_k$ to an integer, we can obtain $l^*_k={ \left \lfloor  \underset{l_k}{\operatorname{arg}}\, \left(\frac{\partial T_{k}}{\partial l_k}=0 \right) \right \rfloor  }.$ 

\section{}
\label{appendix:proof of lemma 1}
\textit{Proof of Lemma 2}: The heterogeneous clients have different local computing power $f_{\tilde{k}}^\text{C}$, and the available computing resources of the PS is limited to $F^\text{max}$. To solve the min-max problem (32), the PS must allocate its computing resources to the clients with larger training latency. Therefore, only part of the $K$ clients can share the computing resources of the PS and can train $\textbf{w}_k$ using the \texttt{SFL}. This part of clients form the group $\mathcal{K}_\texttt{SFL}$, and the left part of the $K$ clients, that can only train $\textbf{w}_k$ independently using the \texttt{FedAvg}, forms the group $\mathcal{K}_\texttt{FedAvg}$. The working mode is like the \textit{water pouring method} [37].

Next, we prove that the clients in group $\mathcal{K}_\texttt{FedAvg}$ consume shorter training time than those in group $\mathcal{K}_\text{\texttt{SFL}}$. Assuming that this assertion holds true, we can index the last client that adopts the \texttt{FedAvg} and the first client that adopts the \texttt{SFL} as $\theta-1$ and $\theta$, respectively. According to eq. (33), we have $T_{\theta-1} \leqslant T_{\theta}$. If not, the PS can take out a part of its computing resources assigned to group $\mathcal{K}_\text{\texttt{SFL}}$ to client $\theta-1$, until the training latency of all the clients in group $\mathcal{K}_\text{\texttt{SFL}} \cup \theta-1$ are the same. Thus, the assertion is proved.

\section{}
\label{appendix:proof of lemma 2}
\textit{Proof of Lemma 3}: From Lemma 2, we know that $T_{\theta-1} \leqslant T_{\theta}$. This means that the clients in group $\mathcal{K}_\texttt{FedAvg}$ consume shorter training times than those in group $\mathcal{K}_\text{\texttt{SFL}}$.

Moreover, when the optimal solution of problem (32) is achieved, the PS allocates its computing resources to the clients in group $\mathcal{K}_\text{\texttt{SFL}}$, so that all the clients in group $\mathcal{K}_\text{\texttt{SFL}}$ are with the same training time $T_{\theta}$. Thus, the lemma is proved.

\section{}
\label{appendix:proof of optimal solution}
\textit{Proof of Lemma 4}: Any client $\tilde{k}$, $\forall \tilde{k} \in \mathcal{K}_{\texttt{FedAvg}}$,
trains $\textbf{w}_{\tilde{k}}$ independently using the \texttt{FedAvg}.
Therefore, $f^\text{S}_{\tilde{k}}=0$, $\forall \tilde{k} \in \mathcal{K}_{\texttt{FedAvg}}$.

For the clients in group $\mathcal{K}_{\texttt{SFL}}$, we take the first and second-order derivatives of $T_{\tilde{k}}$ w.r.t $f^\text{S}_{\tilde{k}}$ using eq. (17) and have
\begin{equation}
    \frac{\partial T_{\tilde{k}}}{\partial f^\text{S}_{\tilde{k}}} = - \frac{ I_{\tilde{k}} \left|\mathcal{B}_{\tilde{k}}\right| (F^\text{S}_{\tilde{k}} + B^\text{S}_{\tilde{k}}) }{ (f^\text{S}_{\tilde{k}})^2 } < 0,\ \forall \tilde{k} \in \mathcal{K}_{\texttt{SFL}}, 
\end{equation}
and
\begin{equation}
    \frac{\partial^2 T_{\tilde{k}}}{\partial (f^\text{S}_{\tilde{k}})^2}  = 2 \frac{ I_{\tilde{k}} \left|\mathcal{B}_{\tilde{k}}\right| (F^\text{S}_{\tilde{k}} + B^\text{S}_{\tilde{k}}) }{ (f^\text{S}_{\tilde{k}})^3} > 0,\ \forall \tilde{k} \in \mathcal{K}_{\texttt{SFL}}.
\end{equation}
This indicates that $T_{\tilde{k}}$ is strictly convex w.r.t $f^\text{S}_{\tilde{k}}$, and there exists a unique root for eq. (17).

By solving eq. (18), we can get the solution of $f^\text{S}_{\tilde{k}}$ ($\forall \tilde{k} \in \mathcal{K}_{\texttt{SFL}}$) as shown in eq. (35)

\section{}
\label{appendix:proof of lemma 3}
\textit{Proof of Lemma 5}: Constraint (36) can be written as the following form.
    \begin{align}
        F^\text{max}= H(T_\theta) = \sum^{\theta}_{\tilde{k}=1} \frac{I_{\tilde{k}} \left|\mathcal{B}_{\tilde{k}}\right| (F^\text{S}_{\tilde{k}} + B^\text{S}_{\tilde{k}}) }{ T_\theta- \frac{ I_{\tilde{k}} \left|\mathcal{B}_{\tilde{k}}\right| (F^\text{C}_{\tilde{k}} + B^\text{C}_{\tilde{k}}) }{ f^\text{C}_{\tilde{k}} } - 2 \frac{I_{\tilde{k}} \left|\mathcal{B}_{\tilde{k}}\right| \Lambda_{\tilde{k}} + |\textbf{w}^\text{C}_{\tilde{k}}|}{r_{\tilde{k}}} }. \label{function H representation about mathcal T}
    \end{align}
The first and second-order derivatives of $H(T_\theta)$ w.r.t to $T_\theta$ are respectively given as
\begin{equation}
    \frac{\partial H }{\partial T_\theta} = \sum^{\theta}_{\tilde{k}=1} \frac{ - I_{\tilde{k}} \left|\mathcal{B}_{\tilde{k}}\right| (F^\text{S}_{\tilde{k}} + B^\text{S}_{\tilde{k}}) }{ \left( T_\theta - \frac{ I_{\tilde{k}} \left|\mathcal{B}_{\tilde{k}}\right| (F^\text{C}_{\tilde{k}} + B^\text{C}_{\tilde{k}}) }{ f^\text{C}_{\tilde{k}} } - 2 \frac{I_{\tilde{k}} \left|\mathcal{B}_{\tilde{k}}\right| \Lambda_{\tilde{k}} + |\textbf{w}^\text{C}_{\tilde{k}}|}{r_{{\tilde{k}}}} \right)^2 }<0,
    \end{equation}
and
\begin{equation}
    \frac{\partial^2 H }{ {\partial (T_\theta)}^2 } = \sum^{\theta}_{{\tilde{k}}=1} \frac{ 2 I_{\tilde{k}} \left|\mathcal{B}_{\tilde{k}}\right| (F^\text{S}_{\tilde{k}} + B^\text{S}_{\tilde{k}}) }{ \left( T_\theta - \frac{ I_{\tilde{k}} \left|\mathcal{B}_{\tilde{k}}\right| (F^\text{C}_{\tilde{k}} + B^\text{C}_{\tilde{k}}) }{ f^\text{C}_{\tilde{k}} } - 2 \frac{I_{\tilde{k}} \left|\mathcal{B}_{\tilde{k}}\right| \Lambda_{\tilde{k}} + |\textbf{w}^\text{C}_{\tilde{k}}|}{r_{{\tilde{k}}}} \right)^3 }>0.
\end{equation}
Therefore, $H(T_\theta)$ has the inverse which is given by $T_\theta = H^{-1}(F^\text{max})$.

The first- and second-order derivatives of $H^{-1}(F^\text{max})$ w.r.t $F^\text{max}$ are given as
\begin{equation}
    \frac{\partial H^{-1} }{\partial F^\text{max} } = \frac{1}{\frac{\partial H }{\partial T_\theta}} < 0,\ \text{and}\ \frac{\partial^2 H^{-1} }{ {\partial (F^\text{max})}^2 } = -\frac{\frac{\partial^2 H }{ {\partial (T_\theta)}^2 }}{{\left(\frac{\partial H }{\partial T_\theta}\right)}^3} > 0,
\end{equation}
respectively. This means that $T_\theta$ is convex w.r.t $F^\text{max}$.

\end{appendices}

\begin{IEEEbiography}[{\includegraphics[width=1in,height=1.25in,clip,keepaspectratio]{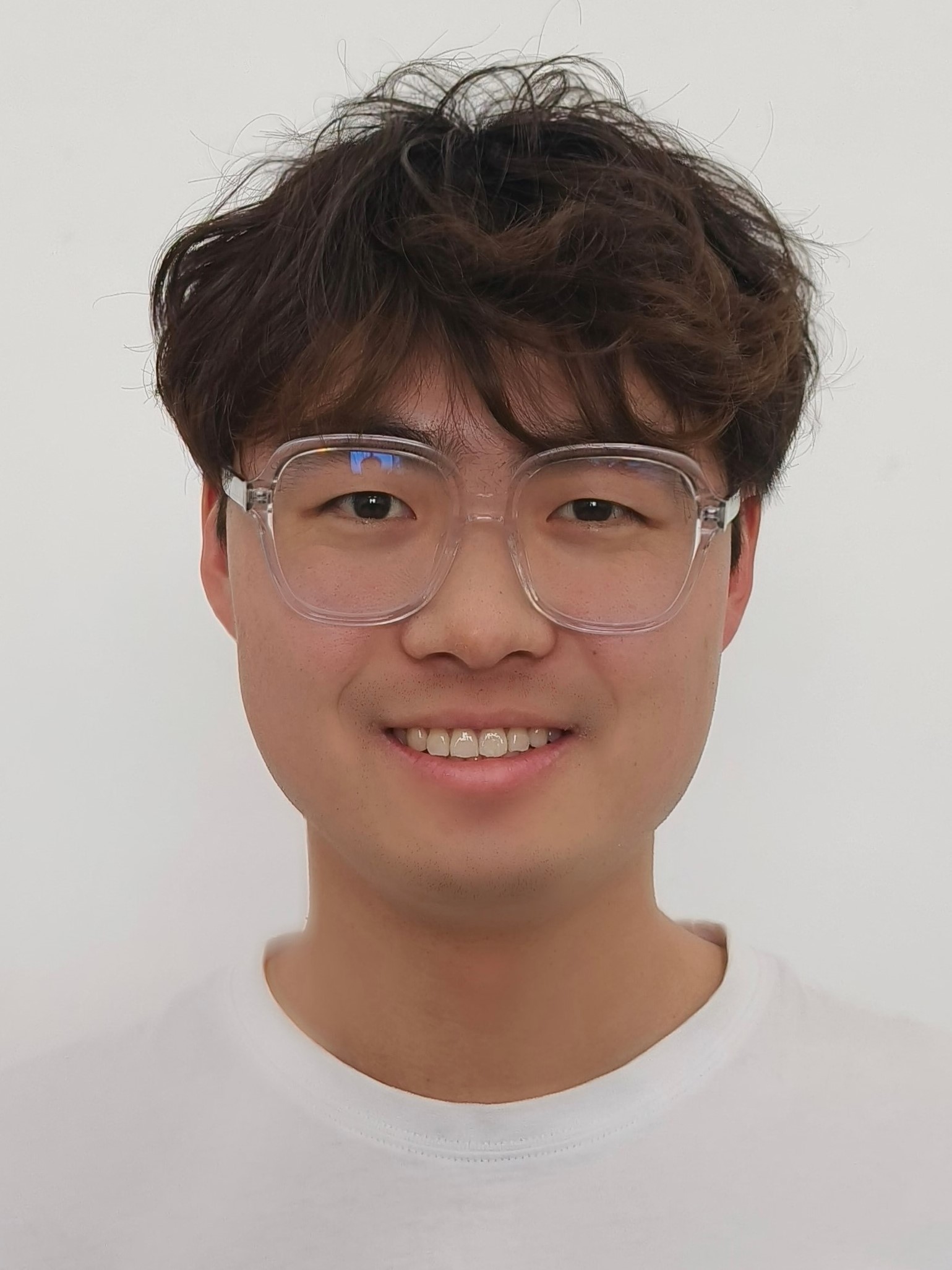}}]{Yao Wen}
    received the bachelor’s degree from the Colloge of Information Science and Technology, Nanjing Forestry University, Nanjing, China, in 2021. He is currently working toward the M.E. degree with the School of Computer Science and Technology, China University of Mining and Technology, Xuzhou, China. His research interests include federated learning, edge computing, deep neural network, split learning, and optimization theory.
\end{IEEEbiography}

\begin{IEEEbiography}[{\includegraphics[width=1in,height=1.25in,clip,keepaspectratio]{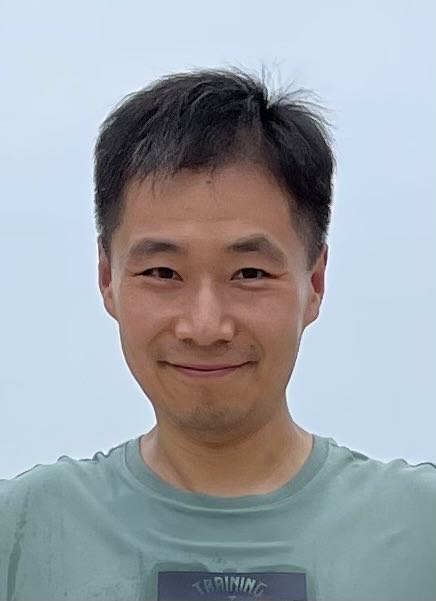}}]{Guopeng Zhang}
    received the bachelor’s degree from the School of Computer Science, Jiangsu Normal University, Xuzhou, China, in 2001, the master’s degree from the School of Computer Science, South China Normal University, Guangzhou, China, in 2005, and the Ph.D. degree from the School of CommunicationEngineering, Xidian University, Xi’an, China, in 2009. He was with ZTE Corporation Nanjing Branch for one year. In 2009, he joined the China University of Minin and Technology, Xuzhou, China, where he is currently a Professor with the School of Computer Science and Technology. He manages research projects funded byvarious sources, such as the National Natural Science Foundation of China. He has authored or coauthored more than 60 journal and conference papers. His main research interests include wireless sensor networks, wireless personalarea networks, and their applications in the Internet of Things.
\end{IEEEbiography}

\begin{IEEEbiography}[{\includegraphics[width=1in,height=1.25in,clip,keepaspectratio]{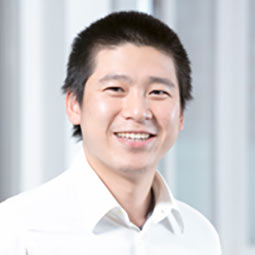}}]{Kezhi Wang}
    received the Ph.D. degree in engineering from the University of Warwick, U.K. He was with the University of Essex and Northumbria University, U.K. Currently, he is a Senior Lecturer with the Department of Computer Science, Brunel University London, U.K. His research interests include wireless communications, mobile edge computing, and machine learning.
\end{IEEEbiography}

\begin{IEEEbiography}[{\includegraphics[width=1in,height=1.25in,clip,keepaspectratio]{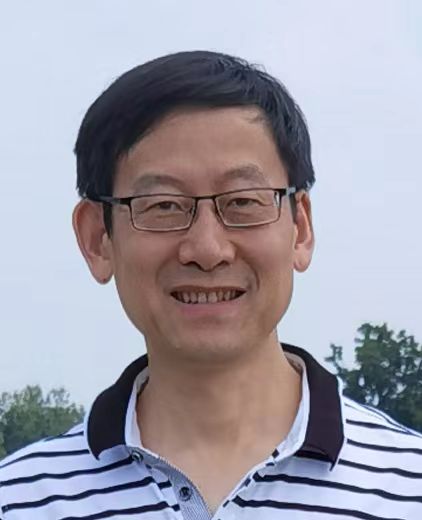}}]{Kun Yang}
    received his PhD from the Department of Electronic \& Electrical Engineering of University College London (UCL), UK. He is currently a Chair Professor in the School of Computer Science \& Electronic Engineering, University of Essex, leading the Network Convergence Laboratory (NCL), UK. He is also an affiliated professor at UESTC, China. Before joining in the University of Essex at 2003, he worked at UCL on several European Union (EU) research projects for several years. His main research interests include wireless networks and communications, IoT networking, data and energy integrated networks and mobile computing. He manages research projects funded by various sources such as UK EPSRC, EU FP7/H2020 and industries. He has published 400+ papers and filed 30 patents. He serves on the editorial boards of both IEEE (e.g., IEEE TNSE, IEEE ComMag, IEEE WCL) and non-IEEE journals (e.g., Deputy EiC of IET Smart Cities). He was an IEEE ComSoc Distinguished Lecturer (2020-2021). He is a Member of Academia Europaea (MAE), a Fellow of IEEE, a Fellow of IET and a Distinguished Member of ACM.
\end{IEEEbiography}

\end{document}